\documentclass[twoside,11pt]{article}

\usepackage{jmlr2e}
\usepackage{url}

\jmlrheading{}{}{}{8/2018}{}{}{Alexander G. de G. Matthews, Mark Rowland, Jiri Hron, Richard E. Turner and Zoubin Ghahramani}

\ShortHeadings{Gaussian Process Behaviour in Wide Deep Neural Networks}{Matthews et al}
\firstpageno{1}

\usepackage{amsmath}
\usepackage{amsfonts}

\usepackage{booktabs} 
\usepackage{multirow}

\usepackage{pifont}
\newcommand{\cmark}{\ding{51}}%
\newcommand{\xmark}{\ding{55}}%

\usepackage{enumerate}

\usepackage{graphicx}
\usepackage{tikz}

\usepackage{subcaption}

\newlength\figureheight
\newlength\figurewidth

\makeatletter
\newcommand*\bigcdot{\mathpalette\bigcdot@{.5}}
\newcommand*\bigcdot@[2]{\mathbin{\vcenter{\hbox{\scalebox{#2}{$\m@th#1\bullet$}}}}}
\makeatother

\newcommand{\equationFullStop}{.}
\newcommand{\equationComma}{,}

\newcommand{\kron}{\delta}

\newcommand{\activationSymbol}{f}
\newcommand{\activitySymbol}{g}
\newcommand{\weightSymbol}{w}
\newcommand{\biasSymbol}{b}

\newcommand{\widthFunction}{h}

\newcommand{\depthSymbol}{\mu}
\newcommand{\smallestDepth}{1}
\newcommand{\secondSmallestDepth}{2}

\newcommand{\widthSymbol}{i}
\newcommand{\widthSymbolB}{j}
\newcommand{\widthSymbolC}{k}
\newcommand{\widthSymbolPair}{\widthSymbol,\widthSymbolB}
\newcommand{\inputSymbol}{x}
\newcommand{\inputSymbolB}{x'}

\newcommand{\indexedObject}[4]{{#1}^{(#2)}_{#3}{#4}}
\newcommand{\indexedActivation}[3]{\indexedObject{\activationSymbol}{#1}{#2}{(#3)}}
\newcommand{\indexedActivity}[3]{\indexedObject{\activitySymbol}{#1}{#2}{(#3)}}
\newcommand{\indexedWeight}[2]{\indexedObject{\weightSymbol}{#1}{#2}{}} 
\newcommand{\indexedStandardNormal}[2]{\indexedObject{\standardNormals}{#1}{#2}{}} 
\newcommand{\indexedBias}[2]{\indexedObject{\biasSymbol}{#1}{#2}{}} 

\newcommand{\activation}{\indexedActivation{\depthSymbol}{\widthSymbol}{\inputSymbol}}
\newcommand{\activity}{\indexedActivity{\depthSymbol}{\widthSymbol}{\inputSymbol}}
\newcommand{\weight}{\indexedWeight{\depthSymbol}{\widthSymbolPair}}
\newcommand{\bias}{\indexedBias{\depthSymbol}{\widthSymbol}}

\newcommand{\vectorize}[1]{#1}
\newcommand{\vectorActivation}{\indexedObject{\vectorize{\activationSymbol}}{2}{}{(\inputSymbol)}}
\newcommand{\vectorActivationB}{\indexedObject{\vectorize{\activationSymbol}}{2}{}{(\inputSymbolB)}}

\newcommand{\dotIndex}{\bigcdot}
\newcommand{\biasVectorZero}{\indexedObject{\vectorize{\biasSymbol}}{1}{}{}}
\newcommand{\biasVectorOne}{\indexedObject{\vectorize{\biasSymbol}}{2}{}{}}
\newcommand{\bigVector}{F}
\newcommand{\bigVectorStandard}{\bigVector^{(\secondSmallestDepth)}}

\newcommand{\nonlinearity}{\phi}
\newcommand{\depth}{D}
\newcommand{\width}{H}
\newcommand{\inputDimension}{M}
\newcommand{\outputDimension}{L}

\newcommand{\normal}{\mathcal{N}}
\newcommand{\varianceSymbol}{C}
\newcommand{\weightVariance}{\varianceSymbol_{\weightSymbol}}
\newcommand{\weightVarianceScaled}{\hat{\varianceSymbol}_{\weightSymbol}}
\newcommand{\weightVarianceDepth}{\weightVariance^{(\depthSymbol)}}
\newcommand{\weightVarianceDepthScaled}{\weightVarianceScaled^{(\depthSymbol)}}
\newcommand{\biasVariance}{\varianceSymbol_{\biasSymbol}}
\newcommand{\biasVarianceDepth}{\biasVariance^{(\depthSymbol)}}
\newcommand{\expectation}[2]{\mathbb{E}_{#2}{\left\lbrack #1\right\rbrack}}
\newcommand{\integrationVariable}{\epsilon}
\newcommand{\integrationVariableB}{\gamma}
\newcommand{\genericCov}{K}

\newcommand{\sequenceVariable}{n}

\newcommand{\featureMapping}{\Phi}
\newcommand{\heaviside}{\Theta}
\newcommand{\rampOrder}{r}

\newcommand{\realLine}{\mathbb{R}}

\newcommand{\MMDnumDimensions}{4}
\newcommand{\MMDnumDataPoints}{10}
\newcommand{\MMDnumNetworks}{2000}
\newcommand{\MMDnumRepeats}{20}
\newcommand{\MMDlengthScale}{1/2}

\newcommand{\characteristicLengthScale}{l}

\newcommand{\assumedWeightVariance}{0.8}
\newcommand{\assumedBiasVariance}{0.2}

\newcommand{\MMD}{\mathcal{MMD}}
\newcommand{\distributionA}{\mathcal{P}}
\newcommand{\distributionB}{\mathcal{Q}}
\newcommand{\testFunction}{h}
\newcommand{\hilbertSpace}{\mathcal{H}}
\newcommand{\hilbertSpaceNorm}[1]{|| #1 ||_{\hilbertSpace}}

\newcommand{\predictiveNumPoints}{10}
\newcommand{\predictiveDimension}{4}

\newcommand{\naturalNumbers}{\mathbb{N}}
\newcommand{\couplingIndex}{n}

\usepackage{thm-restate}

\newcommand{\datapoint}{x}
\newcommand{\datapointi}{\datapoint[i]}

\newcommand{\indexSet}{\mathcal{X}}

\newcommand{\genericReal}{u}
\newcommand{\envelopegradient}{m}
\newcommand{\envelopeconstant}{c}

\newcommand{\sequenceSpace}{\realLine^{\countableIndexSet}}

\newcommand{\sequenceMetric}{\rho}
\newcommand{\sequenceIndex}{j}

\newcommand{\genericRV}{X}
\newcommand{\genericRVB}{Y}

\newcommand{\rowIndex}{n}
\newcommand{\rowVariance}{\sigma^2_{\rowIndex}}
\newcommand{\rowStd}{\sigma_{\rowIndex}}
\newcommand{\limitStd}{\sigma_{*}}
\newcommand{\Var}{\text{Var}}
\newcommand{\Cov}{\text{Cov}}
\newcommand{\Prob}{\text{Pr}}
\newcommand{\limitVariance}{\sigma^2_{*}}

\newcommand{\colIndex}{i}

\newcommand{\generalSum}{S}
\newcommand{\littleO}{\mathrm{o}}

\newcommand{\standardNormals}{\epsilon}
\newcommand{\summand}{\gamma}

\newcommand{\projection}{\mathcal{T}}
\newcommand{\projectionIndeces}{\mathcal{L}}
\newcommand{\projectionCoefficients}{\alpha}
\newcommand{\atWidth}{[\sequenceVariable]}
\newcommand{\atLimit}{[*]}
\newcommand{\monotoneCLTfunc}{h}
\newcommand{\projectionVar}{\sigma^{2}(\depthSymbol,\projectionIndeces,\projectionCoefficients)\atWidth}
\newcommand{\limitVar}{\sigma^{2}(\depthSymbol,\projectionIndeces,\projectionCoefficients)\atLimit}

\newcommand{\projectionIndex}{\alpha}

\newcommand{\countableIndexSet}{Q}
\newcommand{\stochasticProcess}{U}
\newcommand{\limitStochasticProcess}{U_{*}}

\begin{document}

\title{\vskip -0.5in Gaussian Process Behaviour in Wide Deep Neural Networks \vskip 0.2in}
\author{Alexander G. de G. Matthews \email am554@cam.ac.uk \\ \addr Department of Engineering \\ Trumpington Street \\ University of Cambridge, UK 
       \AND Mark Rowland \email mr504@cam.ac.uk \\ \addr Department of Pure Mathematics and Mathematical Statistics\\ Wilberforce Road\\ University of Cambridge, UK \AND Jiri Hron  \email jh2084@cam.ac.uk \\ \addr Department of Engineering\\ Trumpington Street\\ University of Cambridge, UK \AND Richard E. Turner \email ret26@cam.ac.uk \\ \addr Department of Engineering\\ Trumpington Street\\ University of Cambridge, UK  \AND Zoubin Ghahramani \email zoubin@eng.cam.ac.uk \\ \addr Department of Engineering\\ Trumpington Street\\ University of Cambridge, UK  \\ Uber AI Labs}

\editor{}

\maketitle

\vskip -0.2in
\begin{abstract}
Whilst deep neural networks have shown great empirical success, there is still much work to be done to understand their theoretical properties. In this paper, we study the relationship between random, wide, fully connected, feedforward networks with more than one hidden layer and Gaussian processes with a recursive kernel definition. We show that, under broad conditions, as we make the architecture increasingly wide, the implied random function converges in distribution to a Gaussian process, formalising and extending existing results by \citet{Neal1996} to deep networks.  To evaluate convergence rates empirically, we use maximum mean discrepancy. We then compare finite Bayesian deep networks from the literature to Gaussian processes in terms of the key predictive quantities of interest, finding that in some cases the agreement can be very close. We discuss the desirability of Gaussian process behaviour and review non-Gaussian alternative models from the literature.\footnote{Code for the experiments in the paper can be found at \url{https://github.com/widedeepnetworks/widedeepnetworks} }
\end{abstract}

\section{Introduction}

This work substantially extends the work of \citet{DeepWideICLR} published at ICLR 2018. Deep feedforward neural networks have emerged as an essential component of modern machine learning. As such there has been significant research effort in trying to understand the theoretical properties of such models. One important branch of this research is the study of random networks. By assuming a probability distribution on the network parameters, a distribution is induced on the input to output function that the networks encode. This has proved important in the study of initialisation and learning dynamics \citep{Schoenholz2017} and expressivity \citep{Poole2016}. It is, of course, essential in the study of Bayesian priors on networks \citep{Neal1996}. The Bayesian approach makes little sense if prior assumptions are not understood, and distributional knowledge can be essential in finding good posterior approximations. 

Since we typically want our networks to have high modelling capacity, it is natural to consider limit distributions of networks as they become large. Whilst distributions on deep networks are generally challenging to work with exactly, the limiting behaviour can lead to more insight. Further, as we shall see, finite networks used in the literature may be very close to this behaviour.

The seminal work in this area is that of \citet{Neal1996}, which showed that under certain conditions random neural networks with one hidden layer converge to a Gaussian process. The question of the type of convergence is non-trivial and part of our discussion. Historically this result was significant because it provided a connection between flexible Bayesian neural networks and Gaussian processes \citep{Williams1998,Rasmussen2006}
 
\subsection{Our contributions}

We extend the theoretical understanding of random fully connected networks and their relationship to Gaussian processes. In particular, we prove a rigorous result (Theorem \ref{thm:main}) on the convergence of certain sequences of finite fully connected networks with more than one hidden layer to Gaussian processes. The number of hidden layers can be any fixed number. The sizes of the hidden layers must strictly increase for each network in the sequence although the different hidden layers are allowed to grow at different rates. The weights are assumed to be independent normally distributed with their variances sensibly scaled as the network grows following the prescription of \citet{Neal1996}. The nonlinearities are assumed to obey the `linear envelope' Condition \ref{definition:envelope} which all commonly used nonlinearities do in fact obey. Since these are the only assumptions on the sequence of networks it will be seen that the result is a meaningfully general one. 

Further, we empirically study the distance between finite networks and their Gaussian process analogues by using maximum mean discrepancy \citep[MMD,][]{Gretton2012} as a distance measure. We then systematically compare exact Gaussian process inference with `gold standard' MCMC inference for finite Bayesian neural networks. Of the six datasets we consider, five show close agreement between the two models. Owing to the computational burden of the MCMC algorithms, the problems we can study by this method are constrained in terms of their network size, the data dimensionality and the number of data points. Nevertheless our results suggest that some experiments in the literature studied under the banner of Bayesian deep learning would have given very similar results to a Gaussian process with the appropriate kernel. A practical recommendation following from our study is that the Bayesian deep learning community should routinely compare their results to Gaussian processes with the kernels studied in this paper.

Our work is of relevance to the theoretical understanding of neural network initialisation and dynamics. It is also important in the area of Bayesian deep networks because it demonstrates that Gaussian process behaviour can arise in more situations of practical interest than previously thought. If this behaviour is desired then Gaussian process inference (exact and approximate) should also be considered in addition to standard techniques for inference in Bayesian deep learning. In some scenarios, the behaviour may not be desired because it implies a lack of a hierarchical representation and a Gaussian statistical assumption. We therefore highlight promising ideas from the literature to prevent such behaviour.

\subsection{Related work}

The case of random neural networks with one hidden layer was studied by \citet{Neal1996}. \citet{Cho2009} provided analytic expressions for single layer kernels including those corresponding to a rectified linear unit (ReLU). They also studied recursive kernels designed to `mimic computation in large, multilayer neural nets'. As discussed in Section \ref{section:specificKernels} they arrived at the correct kernel recursion through an erroneous argument. Such recursive kernels were later used with empirical success in the Gaussian process literature \citep{Krauth2017}, with a similar justification to that of Cho and Saul. The first case we are aware of using a Gaussian process construction with more than one hidden layer is the work of \citet{Hazan2015}. Their contribution is similar in content to Lemma \ref{lemma:recursion} discussed here, and the work has had increasing interest from the kernel community \citep{Mitrovic2017}. Recent work from \citet{Daniely2016} uses the concept of `computational skeletons' to give concentration bounds on the difference in the second order moments of large finite networks and their kernel analogue, with strong assumptions on the inputs. The Gaussian process view given here, without strong input assumptions, is related but concerns not just the first two moments of a random network but the full distribution. As such the theorems we obtain are distinct. A less obvious connection is to the recent series of papers studying deep networks using a mean field approximation \citep{Poole2016,Schoenholz2017}. In those papers a second order approximation gives equivalent behaviour to the kernel recursion. By contrast, in this paper the claim is that the behaviour emerges as a consequence of increasing width and is therefore something that needs to be proved. Another surprising connection is to the analysis of self-normalizing neural networks \citep{Klambauer2017}. In their analysis the authors assume that the hidden layers are wide in order to invoke the central limit theorem. The premise of the central limit theorem will only hold approximately in layers after the first one and this theoretical barrier is something we discuss here. An area that is less related than might be expected is that of `Deep Gaussian Processes' (DGPs) \citep{Damianou2013}. As will be discussed in Section \ref{section:avoidingGPs}, narrow intermediate representations mean that the marginal behaviour of DGPs is not close to that of a Gaussian process. \citet{Duvenaud2014} offer an analysis that largely applies to DGPs though they also study the Cho and Saul recursion with the motivating argument from the original paper. 

Simultaneously with the submission of the previous version of our paper to ICLR 2018, at the same conference venue, \citet{Lee2018} released a paper that has overlap with our own. There are however some important differences. Empirically, whilst we compare finite Bayesian neural networks, using `gold standard', asymptotically exact, sampling and MMD, to their Gaussian process analogues, Lee et al.\ compare finite neural networks trained with stochastic gradient descent (SGD) to Gaussian processes instead. The latter comparison to SGD is suggestive that this optimization method mimics Bayesian inference -- an idea that has been receiving increasing attention \citep{Welling2011,Mandt2017,Smith2018}. This is of particular importance because typically SGD is still more scalable than traditional Markov Chain based methods, enabling Lee et al.\ to consider some relatively large datasets. The empirical comparison of Lee et al.\ is therefore particularly intriguing and we hope it will lead to follow up work. Therefore, whilst there is overlap, the two papers also have independent value going forward. Theoretically, there are also important differences. Lee et al.\ give an argument for the Gaussian process limit, although importantly this depends on sequentially taking the number of units in each successive layer to be infinite. The proof we give here concerns the case where the layers grow simultaneously, which is arguably more relevant in practice. Note that we show a precise type of convergence -- namely `weak convergence' also called `convergence in distribution'. Owing to the challenging nature of obtaining a~full rigorous proof, the earlier version of this paper \citep{DeepWideICLR} did not achieve full generality either. We needed to assume specific growth rates for the sizes of the hidden layers, 
and the ReLU nonlinearity. What follows here removes these assumptions and thus resolves the conjecture made in the earlier version of this work in the affirmative. The new proof method, placing a particular emphasis on exchangeability, may well be of use more generally.

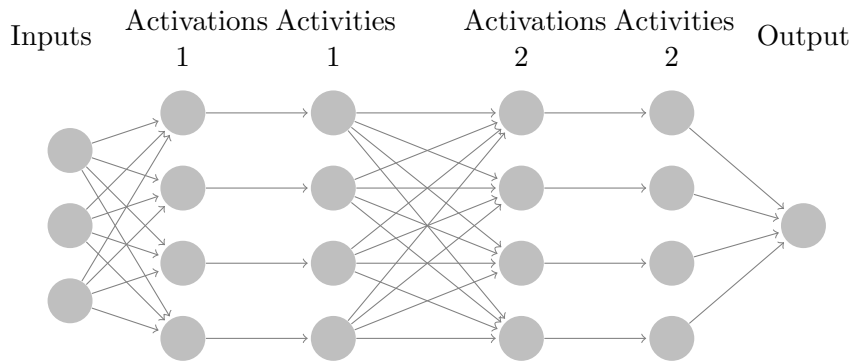
\begin{figure*}[h]
\setlength\figureheight{0.175\textheight}
\setlength\figurewidth{0.8\textwidth}
\centering\begin{tikzpicture}[shorten >=1pt,->,draw=black!50, node distance=1.75cm]
    \tikzstyle{every pin edge}=[<-,shorten <=1pt]
    \tikzstyle{neuron}=[circle,fill=black!25,minimum size=17pt,inner sep=0pt]
    \tikzstyle{input neuron}=[neuron, fill=green!50];
    \tikzstyle{output neuron}=[neuron, fill=red!50];
    \tikzstyle{hidden neuron}=[neuron, fill=blue!50];
    \tikzstyle{annot} = [text width=4em, text centered]

    \foreach \name / \y in {1,...,3}
        \node[neuron] (I-\name) at (0,-\y cm) {};

    \foreach \name / \y in {1,...,4}
        \path[yshift=0.5cm]
            node[neuron] (H1-\name) at (1.5,-\y cm) {};
            
    \foreach \name / \y in {1,...,4}
	    \path[yshift=0.5cm]
		    node[neuron] (A1-\name) at (3.5,-\y cm) {};
            
    \foreach \name / \y in {1,...,4}
	    \path[yshift=0.5cm]
		    node[neuron] (H2-\name) at (6.0,-\y cm) {};
		    
    \foreach \name / \y in {1,...,4}
	    \path[yshift=0.5cm]
		    node[neuron] (A2-\name) at (8.0,-\y cm) {};

    \foreach \name / \y in {1,...,1}
	    \path[yshift=-1cm]
		    node[neuron] (O-\name) at (9.75,-\y cm) {};
    

    \foreach \source in {1,...,3}
        \foreach \dest in {1,...,4}
            \path (I-\source) edge (H1-\dest);
            
    \foreach \source in {1,...,4} 
	    \path (H1-\source) edge (A1-\source);
            
    \foreach \source in {1,...,4}
	    \foreach \dest in {1,...,4}
		    \path (A1-\source) edge (H2-\dest);
		    
    \foreach \source in {1,...,4} 
	    \path (H2-\source) edge (A2-\source);

    \foreach \source in {1,...,4}
	    \foreach \dest in {1,...,1}
	        \path (A2-\source) edge (O-\dest);

    \node[annot,above of=H1-1, node distance=1cm] (hl1) {Activations 1};
    \node[annot,above of=A1-1, node distance=1cm] (al1) {Activities 1};
    \node[annot,above of=H2-1, node distance=1cm] (hl2) {Activations 2};
    \node[annot,above of=A2-1, node distance=1cm] (al2) {Activities 2};
    \node[annot,left of=hl1] {Inputs};
    \node[annot,right of=al2] {Output};
\end{tikzpicture}
\vspace{3mm}
\caption{\label{fig:neural_net_diagram}
In this paper we consider fully connected feedforward networks with more than one hidden layer. We call the pre-nonlinearity an \emph{activation} and post-nonlinearity an \emph{activity}. As the network becomes increasingly wide the distribution of the marginal distributions of the activations at each layer and of the output will become close to a Gaussian process in a sense described in the text.
}
\end{figure*}

\section{The deep wide limit}

\subsection{The result for one hidden layer}

We consider a fully connected network as shown in Figure \ref{fig:neural_net_diagram}. The inputs and outputs will be real-valued vectors of dimension $\inputDimension$ and $\outputDimension$ respectively. The network is fully connected.  The initial step and recursion are standard. The initial step is: 

\begin{align}\label{eq:initial}
\indexedActivation{\smallestDepth}{\widthSymbol}{\inputSymbol} = \sum_{\widthSymbolB=1}^{\inputDimension} \indexedWeight{\smallestDepth}{\widthSymbolPair} \inputSymbol_{\widthSymbolB} + \indexedBias{\smallestDepth}{\widthSymbol} \, .
\end{align}

\noindent We make the functional dependence on $\inputSymbol$ explicit in our notation as it will help clarify what follows. For a network with $D$ hidden layers the recursion is, for each $\mu=1,\ldots,D$,

\begin{align}\label{eq:recursion}
\activity &= \nonlinearity( \indexedActivation{\depthSymbol}{\widthSymbol}{\inputSymbol} ) \, ,  \hspace{5 pt} \\
\indexedActivation{\depthSymbol+1}{\widthSymbol}{\inputSymbol}
 &= \sum_{\widthSymbolB=1}^{\width_\depthSymbol} \indexedWeight{\depthSymbol+1}{\widthSymbol, \widthSymbolB} \indexedActivity{\depthSymbol}{\widthSymbolB}{\inputSymbol} + \indexedBias{\depthSymbol+1}{\widthSymbol}  \, , \label{eq:activationDef}
\end{align}
so that $f^{(D+1)}(x)$ is the output of the network given input $x$. $\nonlinearity$ denotes the nonlinearity. In all cases the equations hold for each value of $\widthSymbol$; $\widthSymbol$ ranges between $1$ and $\width_{\depthSymbol}$ in Equation \eqref{eq:recursion}, and between $1$ and $\width_{\depthSymbol+1}$ in Equation \eqref{eq:activationDef} except in the case of the final activation where the top value is $\outputDimension$. The network could of course be modified to be probability simplex-valued by adding a softmax at the end.

A distribution on the parameters of the network will be assumed. Conditional on the inputs, this induces a distribution on the activations and activities. In particular we will assume independent normal distributions on the weights and biases

\begin{align}\label{eq:dist}
\weight &\sim \normal(0, \weightVarianceDepth)  \hspace{5 pt} \text{i.i.d.} \\
\bias &\sim \normal(0, \biasVarianceDepth ) \hspace{5 pt} \text{i.i.d.} \equationFullStop
\end{align}

We will be interested in the behaviour of this network as the widths $\width_\depthSymbol$ becomes large. The weight variances for $\depthSymbol \geq \secondSmallestDepth  $ will be scaled according to the width of the network to avoid a divergence in the variance of the activities in this limit. As will become apparent, the appropriate scaling is 

\begin{align}
\weightVarianceDepth  = \frac{\weightVarianceDepthScaled}{\width_{\depthSymbol-1}} \, , \hspace{12 pt} \depthSymbol \geq \secondSmallestDepth \, .
\end{align}

The assumption is that $\weightVarianceDepthScaled$ will remain fixed as we take the limit. \citet{Neal1996} analysed this problem for $\depth=1$, showing that as $\width_1 \to \infty$, the values of $\indexedActivation{\secondSmallestDepth}{\widthSymbol}{\inputSymbol}$, the output of the network in this case, converge to a certain multi-output Gaussian process if the activities have bounded variance.

Since our approach relies on the multivariate central limit theorem, we will arrange the relevant terms into (column) vectors to make the linear algebra clearer. Consider any two inputs $\inputSymbol$ and $\inputSymbolB$ and all output functions ranging over the index $\widthSymbol$. We define the vector $\vectorActivation$ of length $\outputDimension$ whose elements are the numbers $\indexedActivation{2}{\widthSymbol}{\inputSymbol}$. We define $\vectorActivationB$ similarly. For the weight matrices defined by $\weight$ for fixed $\depthSymbol$ we use a `placeholder' index $\dotIndex$ to return column and row vectors from the weight matrices. In particular $\indexedWeight{\smallestDepth}{\widthSymbolB,\dotIndex}$ denotes row $\widthSymbolB$ of the weight matrix at depth $\smallestDepth$. Similarly, $\indexedWeight{\secondSmallestDepth}{\dotIndex,\widthSymbolB}$ denotes column $\widthSymbolB$ at depth $\secondSmallestDepth$. The biases are given as column vectors $\biasVectorZero$ and $\biasVectorOne$. Finally we concatenate the two vectors $\vectorActivation$ and $\vectorActivationB$ into a single column vector $\bigVectorStandard$ of size $2 \outputDimension$. The vector in question takes the form

\begin{equation}
\bigVectorStandard = \begin{pmatrix} \vectorActivation \\ \vectorActivationB \end{pmatrix} = \begin{pmatrix} \biasVectorOne \\ \biasVectorOne \end{pmatrix} + \sum_{\widthSymbolB=1}^{\width_1} \begin{pmatrix} \indexedWeight{\secondSmallestDepth}{\dotIndex,\widthSymbolB} \nonlinearity( \indexedWeight{\smallestDepth}{\widthSymbolB,\dotIndex} \inputSymbol + \indexedBias{\smallestDepth}{\widthSymbolB}) \\ \indexedWeight{\secondSmallestDepth}{\dotIndex,\widthSymbolB} \nonlinearity(\indexedWeight{\smallestDepth}{\widthSymbolB,\dotIndex} \inputSymbolB + \indexedBias{\smallestDepth}{\widthSymbolB}) \end{pmatrix} \, .
\end{equation}

The benefit of writing the relation in this form is that the applicability of the multivariate central limit theorem is immediately apparent. Each of the vector terms on this right hand side is independent and identically distributed conditional on the inputs $\inputSymbol$ and $\inputSymbolB$. By assumption, the activities have bounded variance. The scaling we have chosen on the variances is precisely that required to ensure the applicability of the theorem, and is also in line with most commonly used initialisation strategies in practice. Therefore as $\width$ becomes large $\bigVectorStandard$ converges in distribution to a multivariate normal distribution. The limiting normal distribution is fully specified by its first two moments. Defining $\integrationVariableB \sim \normal(0, \biasVariance^{(\smallestDepth)}), \integrationVariable \sim \normal(0, \weightVariance^{(\smallestDepth)} I_\inputDimension)$, the moments in question are:

\begin{align}\label{eq:moments}
\expectation{\indexedActivation{\secondSmallestDepth}{\widthSymbol}{\inputSymbol}}{} &= 0 \\
\expectation{\indexedActivation{\secondSmallestDepth}{\widthSymbol}{\inputSymbol} \indexedActivation{\secondSmallestDepth}{\widthSymbolB}{\inputSymbolB}}{} &= \kron_{\widthSymbolPair} \left[ \weightVarianceScaled^{(\secondSmallestDepth)} \expectation{\nonlinearity( \integrationVariable^{T} \inputSymbol + \integrationVariableB ) \nonlinearity( \integrationVariable^{T} \inputSymbolB + \integrationVariableB )}{\integrationVariable, \integrationVariableB} + \biasVariance^{(\secondSmallestDepth)} \right] \, . \label{eq:secondmoments}
\end{align}

Note that we could have taken a larger set of input points to give a larger vector $\bigVector$ and again we would conclude that this vector converged in distribution to a multivariate normal distribution. More formally, we can consider the set of possible inputs as an \emph{index set}. What we have shown is that for any finite index set the distribution over functions converges to a multivariate normal. If we consider these limiting multivariate normals they obey a~consistency property under marginalization. This means that the limiting distributions can be used to define a Gaussian process by the Kolmogorov extension theorem. 

\subsection{Definition of weak convergence of random functions}\label{section:measureConverge}

There are some important technical issues here that are not discussed in the original work of \citet{Neal1996}. In some sense, the convergence of the finite-dimensional distributions is enough if we want to answer questions about finite events, just as many of the uses of Gaussian processes within machine learning \citep{Rasmussen2006} can be expressed in terms of finite-dimensional multivariate normal distributions. The reader who is content with restricting their attention to such a case may safely omit the rest of this subsection.

Given a consistent set of finite-dimensional marginals, the Kolmogorov extension theorem ensures the existence of an underlying infinite-dimensional object -- a distribution over functions. If we want to make precise mathematical statements about convergence to this object some care is needed. 

Firstly, the Kolmogorov theorem ensures the existence of a distribution which is uniquely defined on a specific $\sigma$-algebra, namely the \emph{product} $\sigma$-algebra. The $\sigma$-algebra defines which events we can assign probabilities to. If we try to consider events outside the $\sigma$-algebra then the rules governing probability distributions (c.f.\ measures) can break down. Secondly, in abstract spaces, the definition of convergence in distribution is necessarily with respect to some \emph{topology}. In everything that follows we will assume that this topology is generated by a metric. We also assume that the index set of the stochastic process is countably infinite. We use the metric $\sequenceMetric$ :

\begin{align}\label{eq:metric}
\sequenceMetric(v, v^\prime) = \sum_{i=1}^\infty 2^{-i} \min(1,  |v_i - v_i^\prime| ) \qquad \forall v, v^\prime \in \mathbb{R}^\mathbb{N} \, ,
\end{align}
	
This metric metrises the product topology of the product of countably many copies of $\mathbb{R}$ with the usual Euclidean topology \citep{Dashti2013}. For such a countable index set, it is sufficient \cite[p.~19]{billingsey68} to prove weak convergence of the finite-dimensional marginals of the process to the corresponding multivariate Gaussian random variables. This is not generally the case if we remove the assumption of a countable index set \cite[p.~19]{billingsey68}.

The restriction to countably infinite index sets means that phenomena that depend on uncountably many indices such as continuity, boundedness and differentiability are not covered by our theory. There is literature extending measures on the product $\sigma$-algebra of an uncountable index set using, for instance, the Kolmogorov continuity theorem. One could then consider proving convergence with respect to the topology in question. We do not do this in this paper but it could certainly be of interest.

\subsection{The recursion lemma and the linear envelope property}

In the case of a multivariate normal distribution a set of variables having a covariance of zero implies that the variables are mutually independent. Looking at Equation \eqref{eq:secondmoments}, we see that the limiting distribution has independence between different components $\widthSymbolPair$ of the output. Combining this with the recursion \eqref{eq:recursion}, we might intuitively suggest that the next layer also converges to a multivariate normal distribution in the limit of large $\width_\depthSymbol$. 

This will indeed be the case assuming that the nonlinearity does not induce heavy tail behaviour.  We give an assumption on the nonlinearity that will be used throughout the~sequel:

\begin{definition}[Linear envelope property for nonlinearities]\label{definition:envelope}
A nonlinearity $\nonlinearity: \realLine \mapsto \realLine$ is said to obey the the linear envelope property if there exist $\envelopeconstant,\envelopegradient \geq 0$ such that the following inequality holds

\begin{align}
|\nonlinearity(\genericReal)| \leq \envelopeconstant + \envelopegradient | \genericReal | \hspace{6 pt} \forall \genericReal \in \realLine \, .
\end{align}

\end{definition}

The majority of commonly used nonlinearities, including the sigmoid, ReLU, ELU, and SeLU nonlinearities have the linear envelope property. Intuitively the linear bounds on the nonlinearity stop it from inducing heavy tail behaviour when a random variable is passed through it. An exponential nonlinearity would not have this property. We could indeed craft a nonlinearity that is designed to violate the linear envelope property and give heavy tail behaviour. Consider, for example, the composition of the Gaussian cumulative density function (CDF) followed by the Cauchy inverse CDF. Passing a standard normal variate through such a function would, by construction, give a Cauchy distributed variable, which has an undefined mean. Whilst it may not be the most general assumption possible for what will follow, the linear envelope assumption rules in most practically used nonlinearities and, as we shall see rules out all nonlinearities for which our theory does not hold. 

Next we state the following lemma, which we attribute to \citet{Hazan2015}:

\begin{lemma}[Normal recursion]\label{lemma:recursion}
If the activations of a previous layer are normally distributed with moments:

\begin{align}
\expectation{\indexedActivation{\depthSymbol-1}{\widthSymbol}{\inputSymbol}}{} &= 0 \\
\expectation{\indexedActivation{\depthSymbol-1}{\widthSymbol}{\inputSymbol} \indexedActivation{\depthSymbol-1}{\widthSymbolB}{\inputSymbolB}}{} &= \kron_{\widthSymbolPair} \genericCov(\inputSymbol,\inputSymbolB) \equationComma
\end{align}

Then under the recursion \eqref{eq:recursion} and as $\width \to \infty $ the activations of the next layer converge in distribution to a normal distribution with moments

\begin{align}
\expectation{\activation}{} &= 0 \\
\expectation{\activation \indexedActivation{\depthSymbol}{\widthSymbolB}{\inputSymbolB}}{} &= \kron_{\widthSymbolPair} \left[\weightVarianceDepthScaled \expectation{\nonlinearity(\integrationVariable_{1})\nonlinearity(\integrationVariable_{2})}{(\integrationVariable_1,\integrationVariable_2) \sim \normal(0,  \genericCov)} + \biasVarianceDepth\right] \, , \label{eq:covExpression}
\end{align}

\noindent where $\genericCov$ is a $2 \times 2$ matrix containing the input covariances. 

\end{lemma}

Unfortunately the lemma is not sufficient to show that the joint distribution of the activations of higher layers converge in distribution to a multivariate normal. This is because for finite $\width$ the input activations do not have a multivariate normal distribution - this is only attained (weakly or in distribution) in the limit. It could be the case that the \emph{rate} at which the limit distribution is attained affects the distribution in subsequent layers. 

Therefore the proof of our main result will require considerably more technical machinery then would be suggested by the recursion in Lemma \ref{lemma:recursion}. We discuss the more general result in the next section.

\subsection{Convergence for more than one hidden layer}

In order to state our theorem we will need one more definition, namely that of a width function:

\begin{definition}[Width functions]\label{def:widthFunction}
For a given fixed input $\couplingIndex \in \naturalNumbers$, a width function $\widthFunction_{\depthSymbol} : \naturalNumbers \mapsto \naturalNumbers$ at depth $\depthSymbol$ specifies the number of hidden units $\width_{\depthSymbol}$ at depth $\depthSymbol$. 
\end{definition}

For a given fixed input $\couplingIndex \in \naturalNumbers$, the set of width functions together fully specify a shape for a fully connected network. In this way, the countable sequence of natural numbers specifies a countable sequence of fully connected networks. We will be interested in the case where each of the width functions tends to infinity.  Note that this includes the case of taking the width functions to be the identity, which gives the case where each hidden layer has the same number of hidden units $\width$ and $\width$ tends \emph{jointly} to infinity rather than taking the limit in sequence. We are now ready to state the main theorem. 

\begin{restatable}{theorem}{thmA}\label{thm:main}
	Consider a random deep neural network of the form in Equations \eqref{eq:initial} and \eqref{eq:recursion} with a continuous nonlinearity obeying the linear envelope condition \ref{definition:envelope}. Then for all sets of strictly increasing width functions $\widthFunction_{\depthSymbol}$ and for any countable input set $(\datapointi)_{i=1}^\infty$, the distribution of the output of the network converges in distribution to a Gaussian process as $\couplingIndex \to \infty$. The Gaussian process has mean function zero and covariance function is given by the recursion Lemma \ref{lemma:recursion}.
\end{restatable}
The convergence in distribution in the statement of the theorem is to be understood in relation to the topology induced by the metric $\rho$ described in Expression \eqref{eq:metric}.
Note the generality allowed by the statement in terms of width functions. We could for instance have width functions growing at very different rates, such as $\widthFunction_{\depthSymbol}(\couplingIndex)=\couplingIndex^{\depthSymbol}$. The special case in which all width functions are the identity is most common in other papers on fully connected networks and is used in the majority of our experiments. It is sufficiently important that we state it as a corollary. 

\begin{corollary}
Consider a random deep neural network of the form in Equations \eqref{eq:initial} and \eqref{eq:recursion} with a continuous nonlinearity obeying the linear envelope condition \ref{definition:envelope} and with common number of hidden units $\width_\depthSymbol=\width$ for each hidden layer $\depthSymbol$.  Then for any countable input set $(\datapointi)_{i=1}^\infty$, the distribution of the output of the network converges in distribution to a Gaussian process as $\width \to \infty$. The Gaussian process has mean function zero and covariance function is as in the recursion Lemma \ref{lemma:recursion}.
\end{corollary}

We postpone the proof of the main theorem until Section \ref{section:proof}. We next look at specific instances of the implied covariance function.

\section{Specific kernels under recursion}\label{section:specificKernels}

\citet{Cho2009} suggest a family of kernels based on a recurrence designed to `mimic computation in large, multilayer neural nets'. It is therefore of interest to see how this relates to deep wide Gaussian processes. A kernel may be associated with a feature mapping $\featureMapping(\inputSymbol)$ such that $\genericCov(\inputSymbol, \inputSymbolB) = \featureMapping(\inputSymbol) \bigcdot \featureMapping(\inputSymbolB) $. Cho and Saul define a recursive kernel through a new feature mapping by compositions such as $\featureMapping( \featureMapping( \inputSymbol ) )$. However this cannot be a legitimate way to create a kernel because such a composition represents a type error. There is no reason to think the output dimension of the function $\featureMapping$ matches the input dimension and indeed the output dimension may well be infinite. Nevertheless, the~paper provides an~elegant solution to a~different task: it derives closed form solution to the~recursion from Lemma~\ref{lemma:recursion} \citep{Hazan2015} for the special case

\begin{equation}
\nonlinearity( u ) = \heaviside( u ) u^{\rampOrder} \hspace{5 pt} \text{for} \hspace{5 pt} \rampOrder = 0,1,2,3  \, ,
\end{equation}

\noindent where $\heaviside$ is the Heaviside step function. Specifically, the~recursive approach of \citet{Cho2009} can be adapted by using the~fact that $u^\top z$ for $z \sim \normal(0, L L^\top)$ is equivalent in distribution to $(L^\top u)^\top \varepsilon$ with $\varepsilon \sim \normal(0, I)$, and by optionally augmenting $u$ to incorporate the~bias. Since $\rampOrder=1$ corresponds to rectified linear units, we apply this analytic kernel recursion in all of our experiments.  


\section{Measuring convergence using maximum mean discrepancy}\label{section:mmd}
In this section we use the kernel based two sample tests of \citet{Gretton2012} to empirically measure the similarity of finite random neural networks to their Gaussian process analogues. The maximum mean discrepancy (MMD) between two distributions $\distributionA$ and $\distributionB$ is defined as:

\begin{equation}
\MMD(\distributionA,\distributionB,\hilbertSpace) := \sup_{\hilbertSpaceNorm{\testFunction} \leq 1} \bigg[ \expectation{\testFunction}{\distributionA} - \expectation{\testFunction}{\distributionB} \bigg] \, ,
\end{equation}

\noindent where $\hilbertSpace$ denotes a reproducing kernel Hilbert space and $\hilbertSpaceNorm{\bigcdot}$ denotes the corresponding norm. It gives the biggest possible difference between expectations of a function under the two distributions under the constraint that the function has Hilbert space norm less than or equal to one. We used the unbiased estimator of squared MMD given in Equation (3) of \citet{Gretton2012}.

In this experiment and where required in what follows, we take weight variance parameters $\weightVarianceDepthScaled = \assumedWeightVariance$ and bias variance $\biasVariance = \assumedBiasVariance$. We took $\MMDnumDataPoints$ standard normal input points in $\MMDnumDimensions$ dimensions and pass them through $\MMDnumNetworks$ independent random neural networks drawn from the distribution discussed in this paper. This was then compared to $\MMDnumNetworks$ samples drawn from the corresponding Gaussian process marginal distribution. The experiment was performed with different numbers of hidden layers, different choices of monotonic width functions (which will be described in the sequel), and network sequence index $\couplingIndex \in \naturalNumbers$ as described in Definition \ref{def:widthFunction}. We repeated each experiment $\MMDnumRepeats$ times which allows us to reduce variance in our results and give a simple estimate of measurement error. The experiments use an RBF kernel for the MMD estimate with lengthscale $\MMDlengthScale$. In order to help give an intuitive sense of the distances involved we also include a comparison between two Gaussian processes with isotropic RBF kernels using the same MMD distance measure. The kernel length scales for this pair of `calibration' Gaussian processes are taken to be $\characteristicLengthScale$ and $2\characteristicLengthScale$, where the characteristic length scale $\characteristicLengthScale = \sqrt{8}$ is chosen to be sensible for the standard normal input distribution on the four dimensional space. Note that there are multiple different uses for kernels in this experiment. The first use is to estimate MMD, the second is for the covariance function of the calibration Gaussian processes and the third use is for the covariance function of the limit Gaussian process. The first and second cases both happen to use the RBF kernel with various length scales, but they should not be confused.

We investigated three choices of strictly increasing width functions, all of which meet the assumptions required by Theorem \ref{thm:main} for convergence in distribution to the corresponding Gaussian process. The identity width function $\widthFunction_{\depthSymbol}(\couplingIndex) = \couplingIndex$ corresponds to the case where all hidden layers are the same size and $\couplingIndex$ may be directly identified with the width of the network. To test a broader variety of the predictions made by the theory we introduced two other width function specifications. What we call the \emph{largest last} width function is given by:

\begin{equation}
\widthFunction_{\depthSymbol}(\couplingIndex) = \couplingIndex \depthSymbol \equationFullStop
\end{equation}

For example, in a three hidden layer neural network, with $\couplingIndex=50$, starting from the layer closest to the inputs, we would have hidden layer sizes $50,100,150$. The \emph{largest first} width function is given by:

\begin{equation}
\widthFunction_{\depthSymbol}(\couplingIndex) = \couplingIndex ( \depth -\depthSymbol + 1)
\end{equation}

For example in a three hidden layer neural network, with $\couplingIndex=50$, starting from the layers closest to the inputs, we would have have hidden layer sizes $150,100,50$. For both the largest first and largest last width functions the sequence index $\couplingIndex$ may be identified with the width of the narrowest hidden layer. 

\begin{figure}[h!]
\includegraphics[width = \textwidth]{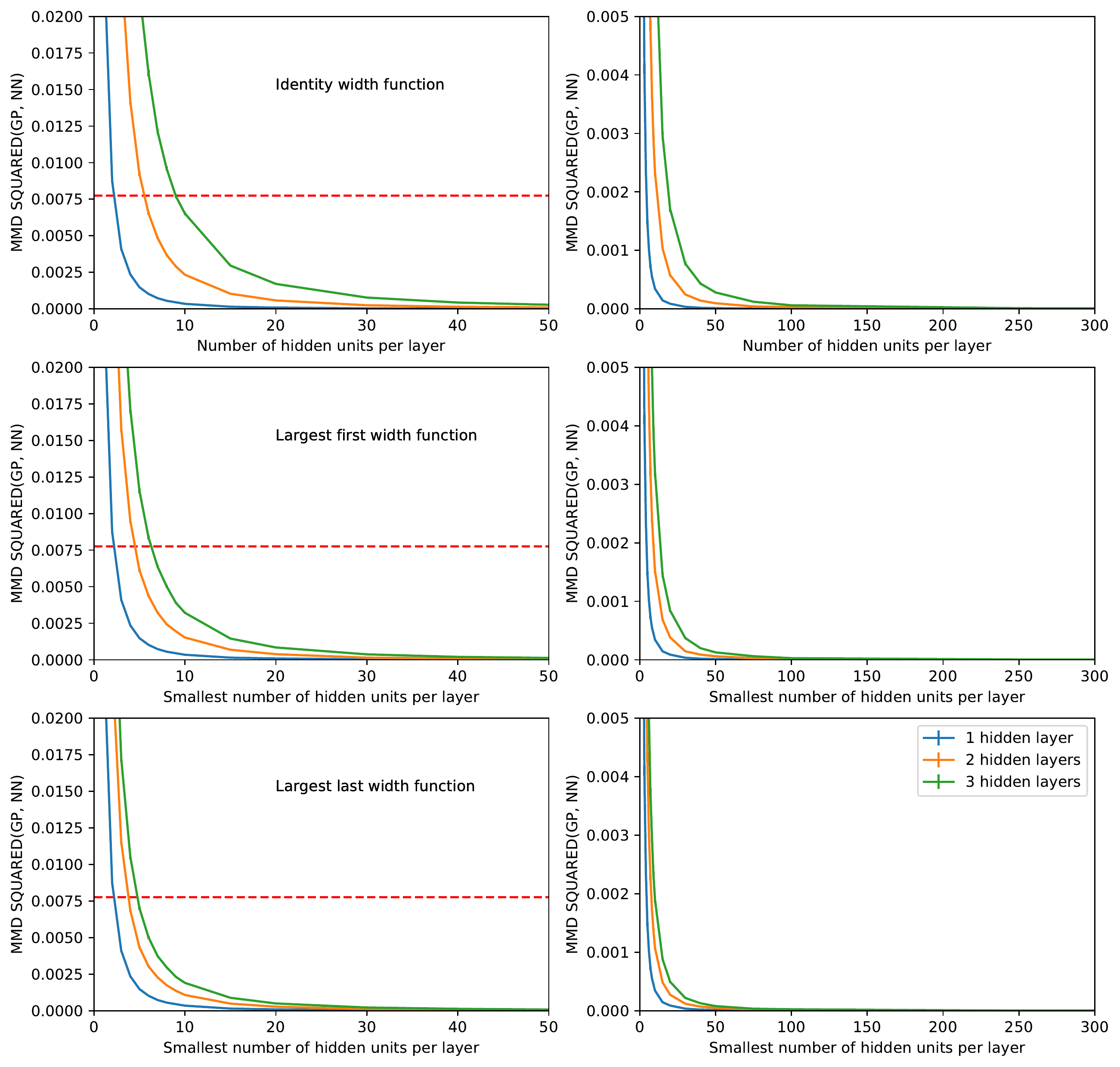}
\caption{\label{fig:mmd} A comparison of finite random neural networks with ReLU nonlinearity to their corresponding Gaussian process analogue using an (RBF) kernel estimator of the squared maximum mean discrepancy (MMD). The results are consistent with the emergence of Gaussian process behaviour as the networks become wide. The red dashed line is for calibration and denotes the squared MMD between two Gaussian processes with isotropic RBF kernels and length scales $\characteristicLengthScale$ and $2\characteristicLengthScale$ where $\characteristicLengthScale=\sqrt{8}$ is the characteristic length scale of the input space. Different columns are different scalings of the same row plots. The rows correspond to different choices of width function. The most standard choice of the same number of hidden units per layer corresponds to the identity width function. The other width functions are described in the text. Assuming all layer sizes are strictly increasing, the independence of the choice of width function is a prediction of the theory, and is consistent with these results.
}
\end{figure}

The results of the experiment are shown in Figure \ref{fig:mmd}. We see that for each fixed depth the network converges towards the corresponding Gaussian process as the width increases. For the same number of hidden units per layer, the MMD distance between the networks and their Gaussian process analogue becomes higher as depth increases. The rate of convergence to the Gaussian process is slower as the number of hidden layers is increased. Unsurprisingly, since the corresponding networks will have strictly more units, both the largest last and largest first width functions converge faster than the identity width function. The largest last width function seems to converge slightly faster than the largest last width function with respect to this metric. The comparison is more interesting in this case since these two width functions have similar numbers of units. All of the results are consistent with the predictions of Theorem \ref{thm:main}.

\section{Empirical Comparison of Bayesian Deep Networks to Gaussian Processes}\label{section:bayesComparison}

In this section we compare the behaviour of finite \emph{Bayesian} deep networks of the form considered in this paper with their Gaussian process analogues. For expectations of bounded continuous functions, if we make the networks wide enough the agreement will be very close. It is also of interest, however, to consider the behaviour of networks actually used in the literature. Fully connected Bayesian deep networks with finite variance priors on the weights have been considered in several recent works \citep{Graves2011,Hernandez2015,Blundell2015,Hernandez2016}, though the specific details vary. From a Bayesian perspective, the previous section could be interpreted as using MMD as a similarity metric between priors. By constrast, in this section we will compare data dependent quantities that are typically used in Bayesian modelling practice.

We use rectified linear units and correct the variances to avoid a loss of prior variance as depth is increased. Our general strategy was to compare exact Gaussian process inference against expensive, `gold standard',  Markov Chain Monte Carlo (MCMC) methods. We choose the latter because used correctly it works well enough to largely remove questions of posterior approximation quality from the calculus of comparison. It does mean however that our empirical study does not extend to datasets which are large in terms of number of data points or dimensionality, where such inference is challenging. We therefore sound a note of caution about extrapolating our empirical finite network conclusions too confidently to this domain. On the other hand, lower dimensional, prior-dominated problems are generally regarded as an area of strength for Bayesian approaches and in this context our results are directly relevant.

We use $3$ hidden layers and $50$ hidden units which is typical of the smaller Bayesian neural networks used by \citet{Hernandez2015}. \citet{Hernandez2015} also use the variance scaling of \citet{Neal1996} on their normally distributed weights and give a hierarchical treatment of the hyperparameters. Note that much larger networks have been used in the literature. For example, \citet{Blundell2015} use as many as $1200$ units per layer, though they use a two component scale mixture of Gaussians for the weight prior. This would require an extension of our theory to non-Gaussian weight distributions for our results to be strictly applicable. Our modest choice of $50$ hidden units per layer is partly also motivated by necessity. For larger networks the MCMC would be prohibitively slow.

The experiments are divided into those with fixed hyperparameters and those where the hyperparameters are learnt. The hyperparameters are specifically the noise variance, the raw weight variance $\weightVarianceScaled$ and the bias variance $\biasVariance$. The latter two hyperparameters are shared across layers. The fixed hyperparameter experiments are the comparison most directly relevant to the theory presented here. However we found that as we moved to larger datasets both the neural network prior and the Gaussian process prior were often misspecified to an extent that made the results practically uninteresting. Since we were already computationally constrained by the neural network MCMC, we adopted the pragmatic solution of using the type II maximum likelihood parameter estimate of the Gaussian process model for both the neural network and Gaussian process priors. Although the number of hyperparameters is small, this technically adds dependency, so the fixed-hyperparameter experiments are complementary. 

\subsection{Experiments with fixed hyperparameters}

We computed the posterior moments by the two different methods on some example datasets. For the MCMC we used Hamiltonian Monte Carlo (HMC) \citep{Neal2010} updates interleaved with elliptical slice sampling \citep{Murray2010}. We considered a simple one-dimensional regression problem and a two dimensional real-valued embedding of the four data point XOR problem. To distinguish this from a later larger embedding we term this the \emph{small XOR} dataset. We see in Figures \ref{fig:one_dimensional_example} and \ref{fig:bias_two_dimensional} (left) that the agreement in the posterior moments between the Gaussian process and the Bayesian deep network is very close. 

A key quantity of interest in Bayesian machine learning is the marginal likelihood. It is the normalising constant of the posterior distribution and gives a measure of the model fit to the data. For a Bayesian neural network, it is generally very difficult to compute, but with care and computational time it can be approximated using Hamiltonian annealed importance sampling \citep{Sohl2012}. The log-importance weights attained in this way constitute a stochastic lower bound on the marginal likelihood \citep{Grosse2015}. Figure \ref{fig:bias_two_dimensional} (right) shows the result of such an experiment compared against the (extremely cheap) Gaussian process marginal likelihood computation on the small XOR problem. The value of the log-marginal likelihood computed in the two different ways agree to within a single nat which is negligible from a model selection perspective \citep{Grosse2015}.

Predictive log-likelihood is a measure of the quality of probabilistic predictions given by a Bayesian regression method on a test point. To compare the two models we sampled $\predictiveNumPoints$  standard normal train and test points in $\predictiveDimension$ dimensions and passed them through a random network of the type under study to get regression targets. We then discarded the true network parameters and compared the predictions of posterior inference between the two methods. We also compared the marginal predictive distributions of a latent function value. Figure \ref{fig:bias_four} shows the results. We see that the correspondence in predictive log-likelihood is close but not exact. Similarly the marginal function values are close to those of a Gaussian process but are slightly more concentrated.

\begin{figure}[h]
\includegraphics[width = \textwidth]{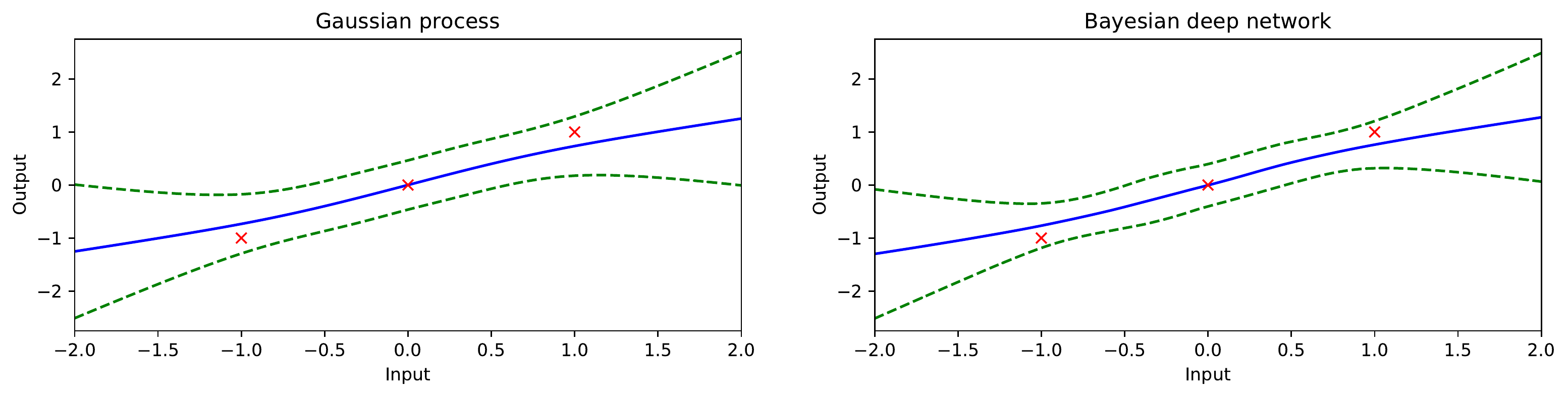}
\caption{\label{fig:one_dimensional_example} A comparison between Bayesian posterior inference in a Bayesian deep neural network and posterior inference in the analogous Gaussian process. The neural network has $3$ hidden layers and $50$ units per layer. The lines show the posterior mean and two $\sigma$ credible intervals.
}
\end{figure}

\begin{figure}[h]
\begin{center}
\begin{subfigure}{0.66\textwidth}
\includegraphics[width = \textwidth]{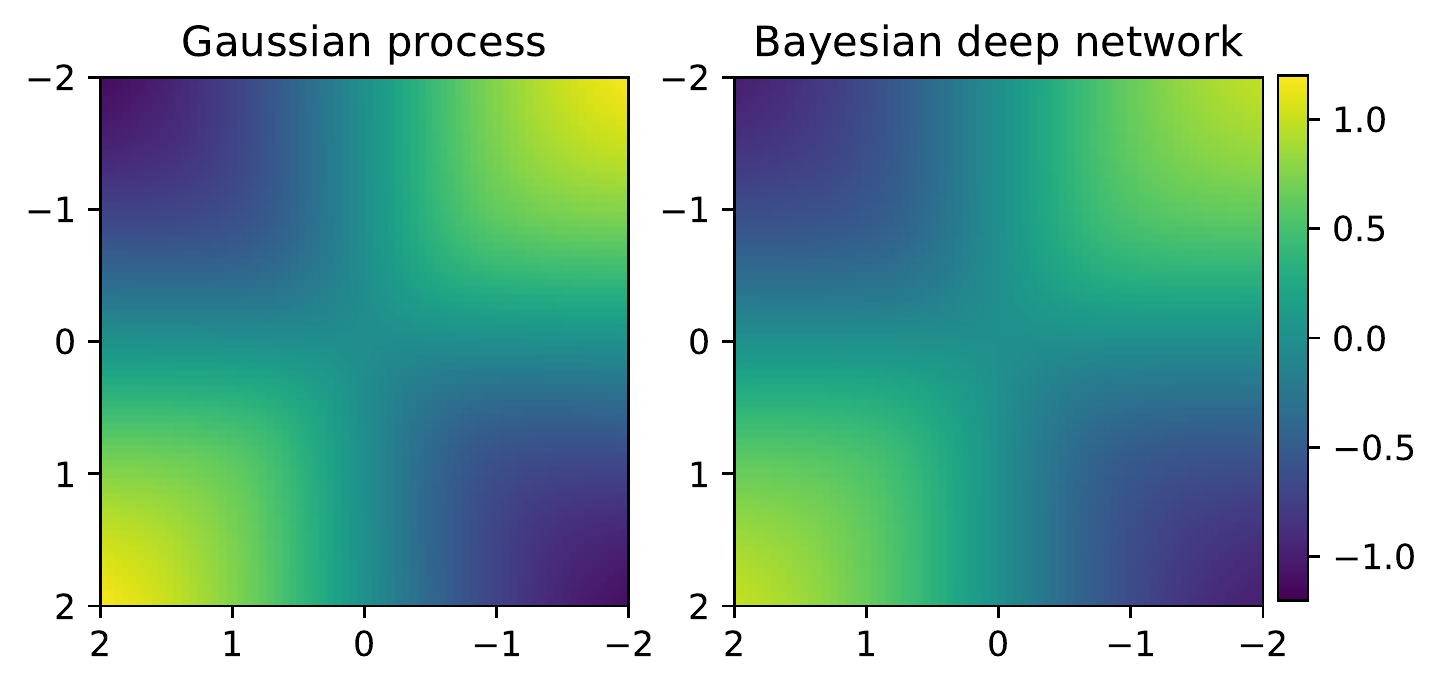}
\end{subfigure}
\begin{subfigure}{0.33\textwidth}
\includegraphics[width = \textwidth]{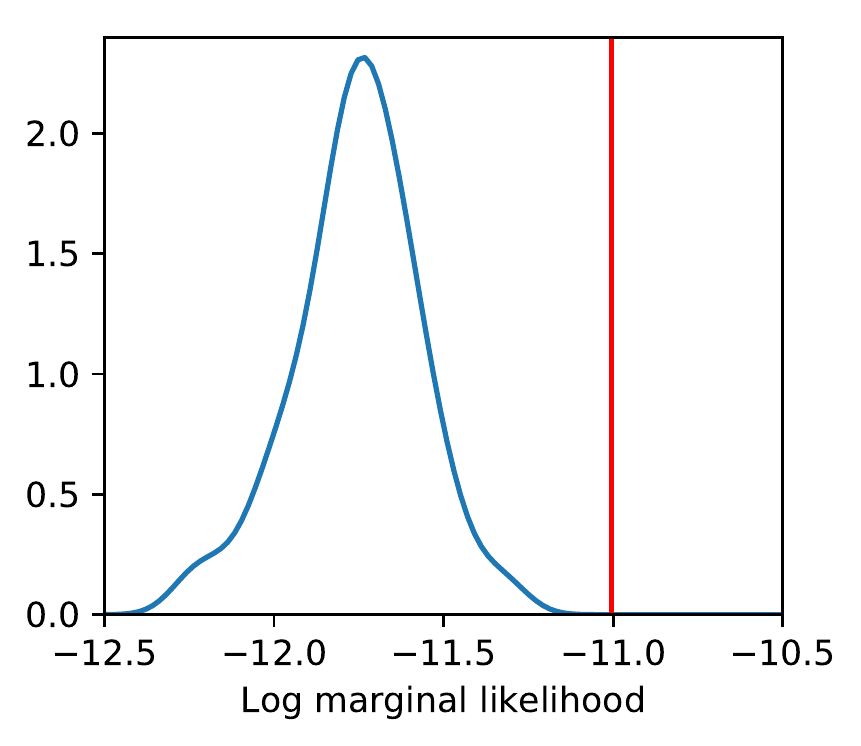}
\end{subfigure}
\end{center}
\caption{\label{fig:bias_two_dimensional} A comparison between posterior inference for a Gaussian process and a Bayesian deep network for the small XOR dataset, a four point real embedding of the XOR function. Left and centre: The two posterior means. The mean absolute different between the two posterior estimate grids is $0.064$. Right: Kernel density estimate of the log weights from annealed importance sampling on a Bayesian deep network compared to the analogous Gaussian process marginal likelihood shown by the vertical line. The neural network has $3$ hidden layers and $50$ units per layer.
}
\end{figure}

\begin{figure}[h]
\begin{center}
\begin{subfigure}{0.49\textwidth}
\includegraphics[width = \textwidth]{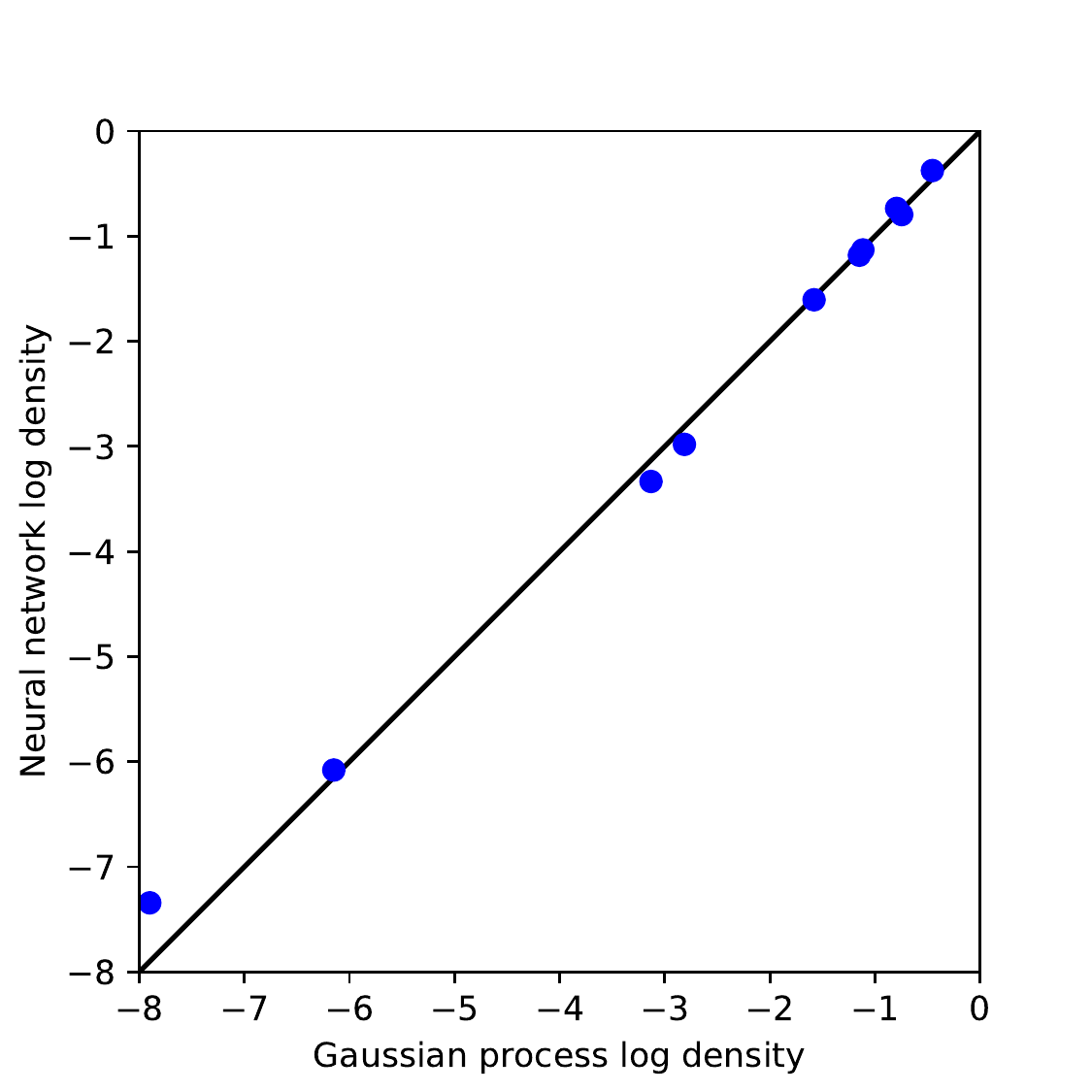}
\end{subfigure}
\begin{subfigure}{0.49\textwidth}
\includegraphics[width = \textwidth]{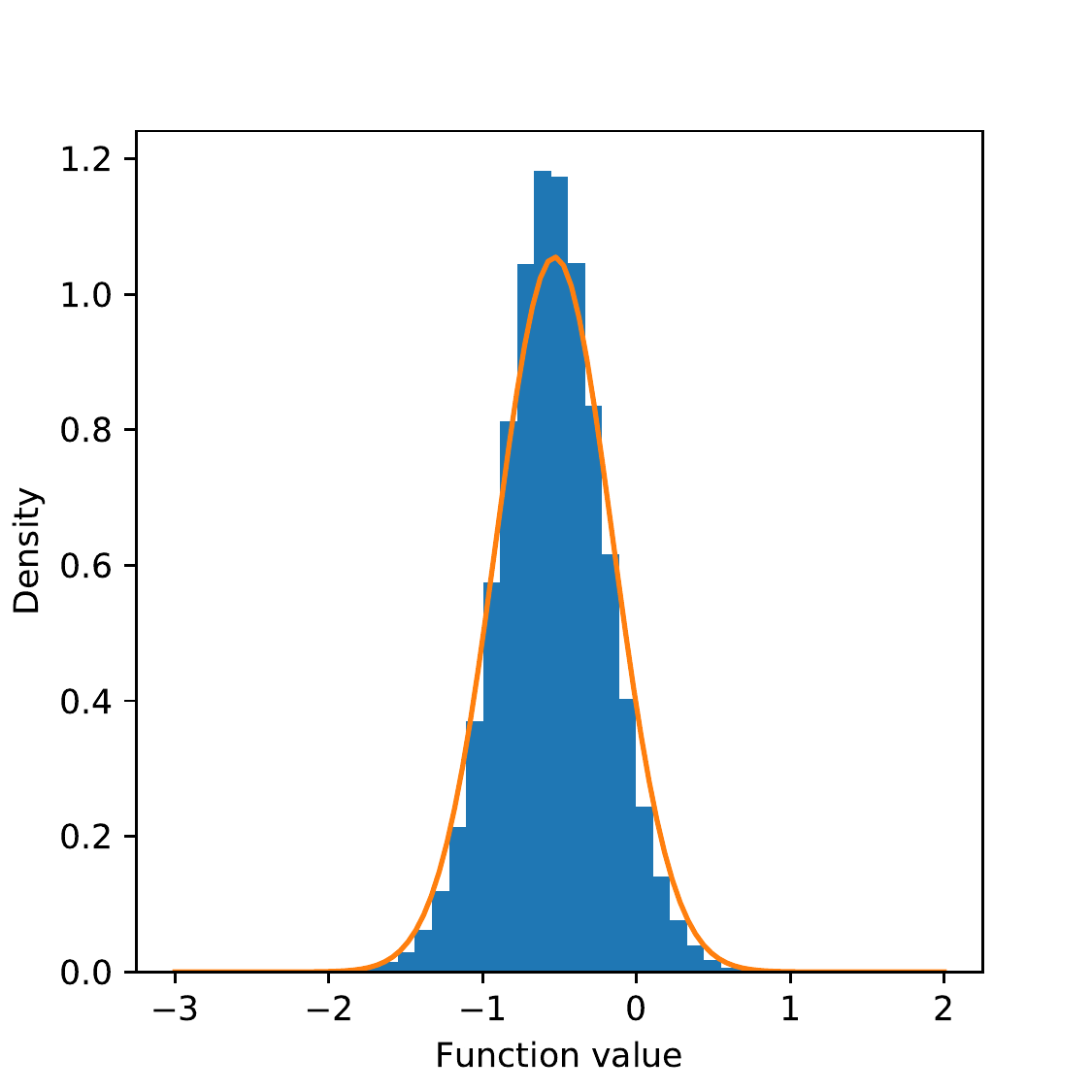}
\end{subfigure}
\end{center}
\caption{\label{fig:bias_four} A comparison of the predictive distributions of a Bayesian deep network and a Gaussian process on a randomly generated test case. Left: the per-point log-densities of the two models. Right: predictive marginal distribution for the latent function on a randomly selected test point.}
\end{figure}

\subsection{Experiments with learnt hyperparameters}

As described above, in this section we compare neural networks and the corresponding Gaussian process on larger datasets using hyperparameters for both models that are taken from the learnt Gaussian process kernel, estimated using type II maximum likelihood.

We made a comparison for the $100$ data point Snelson dataset, a regression benchmark commonly used in the sparse Gaussian process literature \citep{Snelson2005}. Figure \ref{fig:snelson} shows that the agreement is very close.

\begin{figure}[h]
\includegraphics[width = \textwidth]{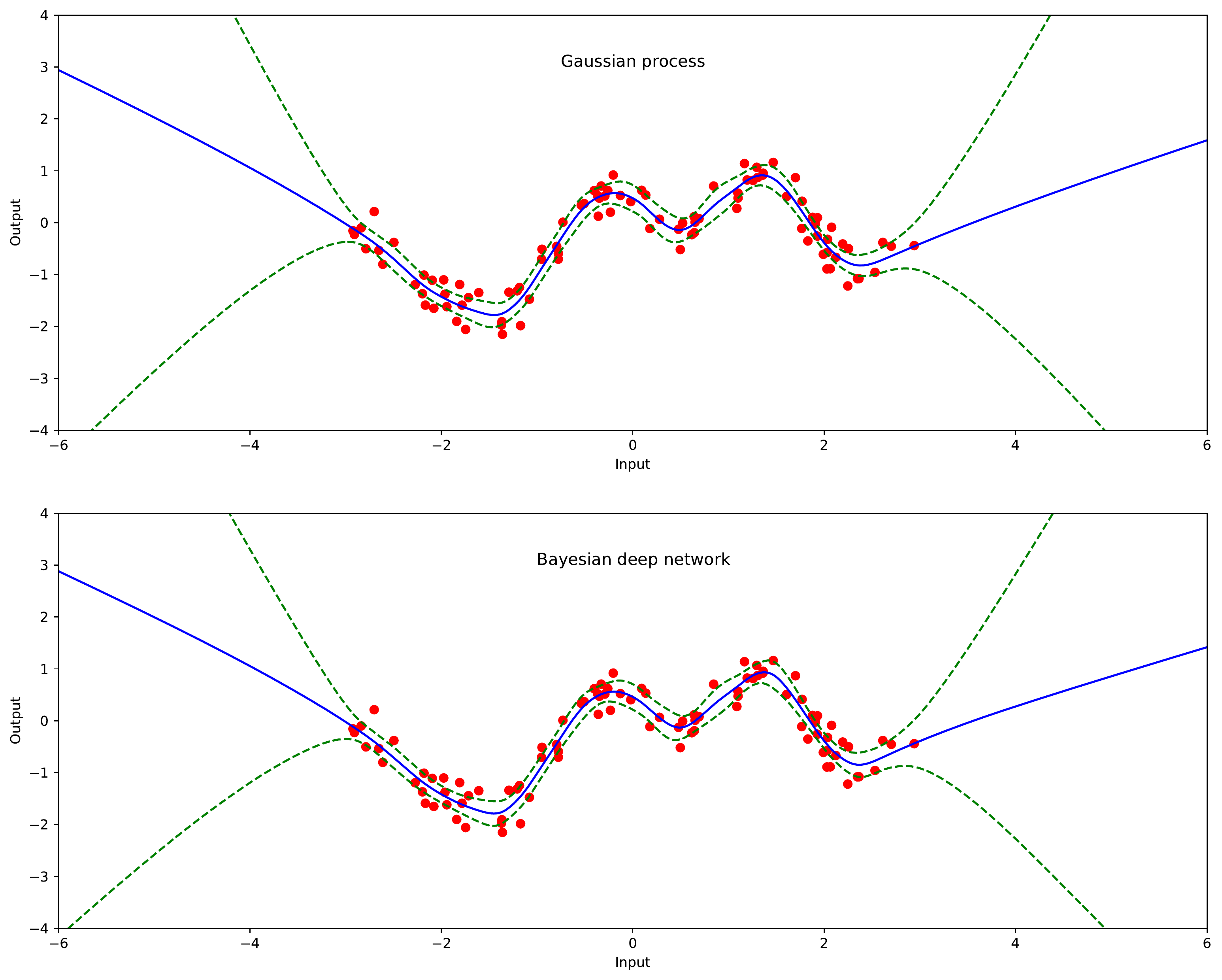}
\caption{\label{fig:snelson} A comparison between Bayesian posterior inference in a Bayesian deep neural network and posterior inference in the analogous Gaussian process for the Snelson dataset. The neural network has $3$ hidden layers and $50$ units per layer. The lines show the posterior mean and two $\sigma$ credible intervals. 
}
\end{figure}

Next we made a comparison for a larger embedding of the real valued XOR function which we term the \emph{smooth XOR dataset}, to distinguish it from the small XOR dataset above. In detail, we have:

\begin{equation}
f(x_1,x_2) = -\gamma x_1 x_2 \exp\left\lbrace - \frac{(x_1^2+x_2^2)}{\beta} \right \rbrace
\end{equation}

\noindent where $\gamma$ and $\beta$ are chosen so that $f(-1,-1)=f(1,1)=-1$ and $f(1,-1) = f(-1,1) = 1$. One hundred input points $(x_1,x_2)$ are sampled from a standard normal distribution and Gaussian noise of variance $0.01$ is added to the outputs. In order to allow better visualisation of the posterior we take test points along two linear cross sections as shown in Figure \ref{fig:smooth_xor}. This allows us to plot the two posteriors along the cross-sections in a manner similar to a one dimensional regression problem. Figure \ref{fig:smooth_xor} shows the results. We can see that there is again close agreement between the Bayesian neural network posterior and that of the Gaussian process.

\begin{figure}[h!!]
\begin{center}
\includegraphics[width = 0.73\textwidth]{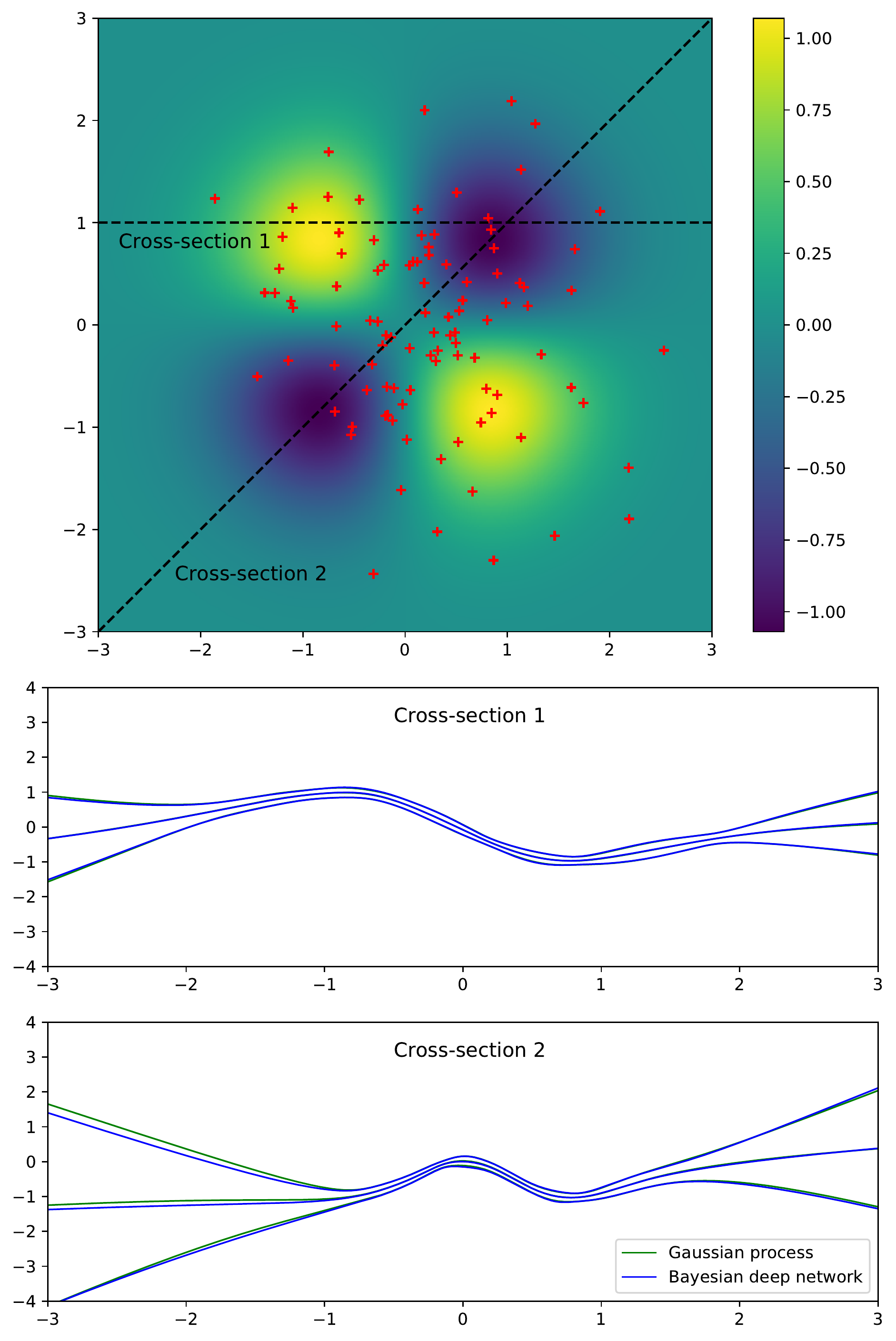}
\end{center}
\caption{\label{fig:smooth_xor} A comparison between the Bayesian posterior in a deep neural network and the analogous Gaussian process for the smooth XOR dataset. Top: A visualization of the smooth XOR dataset. The heat plot shows the smooth XOR function. The red crosses show the position of the training inputs. The black dashed lines show linear cross sections of the space along which we study the two posteriors. Middle and bottom: The two posteriors along the linear cross sections. In each case, the middle line is the posterior mean and the other lines represent the two $\sigma$ credible intervals. 
}
\end{figure}

Finally, we make a comparison on the Delft yacht hydrodynamics dataset. The task is to predict the residuary resistance per unit weight of displacement for a yacht hull based on six relevant attributes. We randomly partition the data into 100 training instances and 208 test instances. The data has very low noise. To make it a more challenging task for probabilistic modelling we add Gaussian noise of variance $0.01$. We evaluate per test data point hold out log likelihood for both the Gaussian process and the neural network and the marginal posterior on a randomly selected test function value. The results are shown in Figure \ref{fig:yacht}. The results indicate that on this dataset the Bayesian deep network and the Gaussian process do not make similar predictions. Of the two, the Bayesian neural network achieves significantly better log likelihoods on average, indicating that a finite network performs better than its infinite analogue in this case. 

\begin{figure}[h!!]
\begin{center}
\begin{subfigure}{0.49\textwidth}
\includegraphics[width = \textwidth]{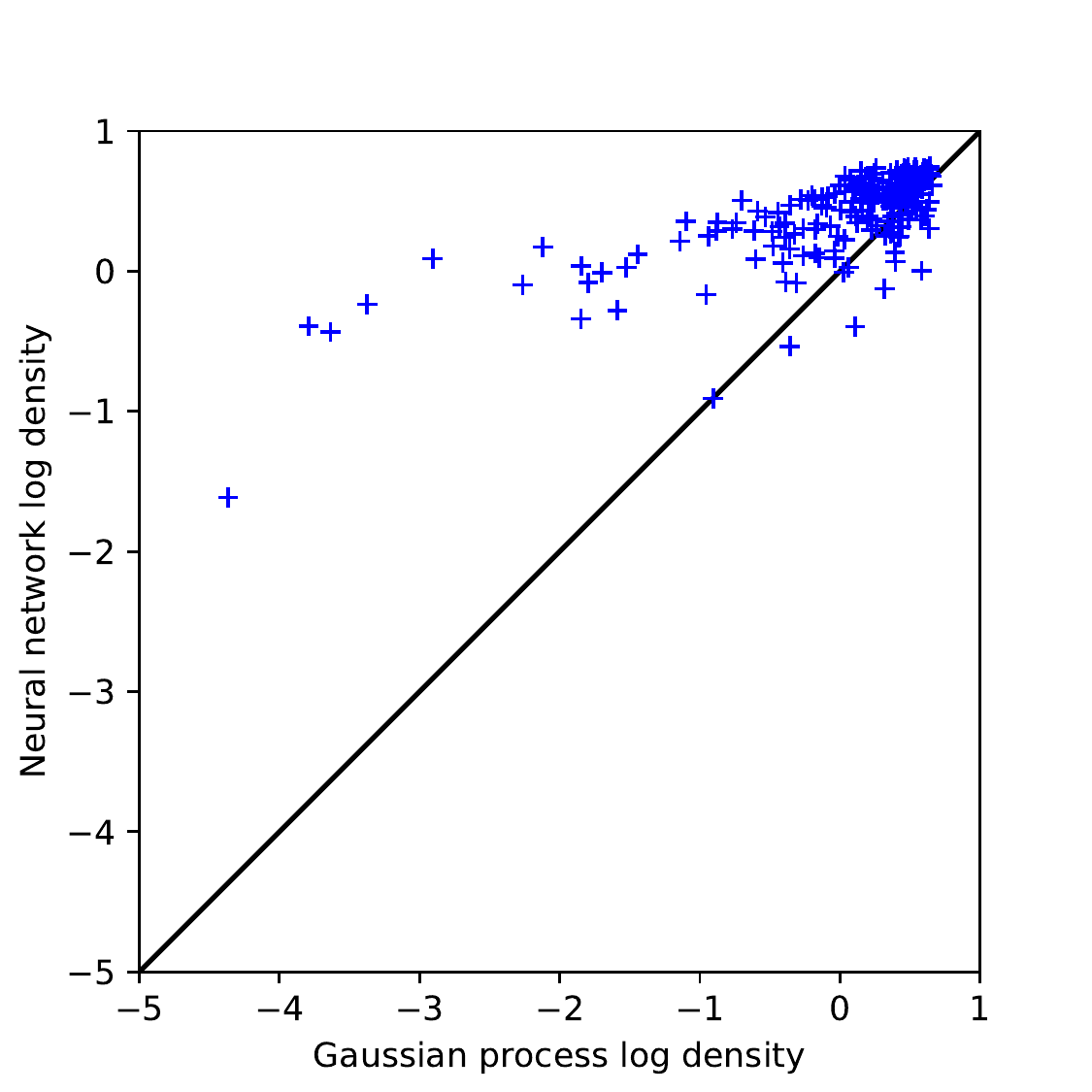}
\end{subfigure}
\begin{subfigure}{0.49\textwidth}
\includegraphics[width = \textwidth]{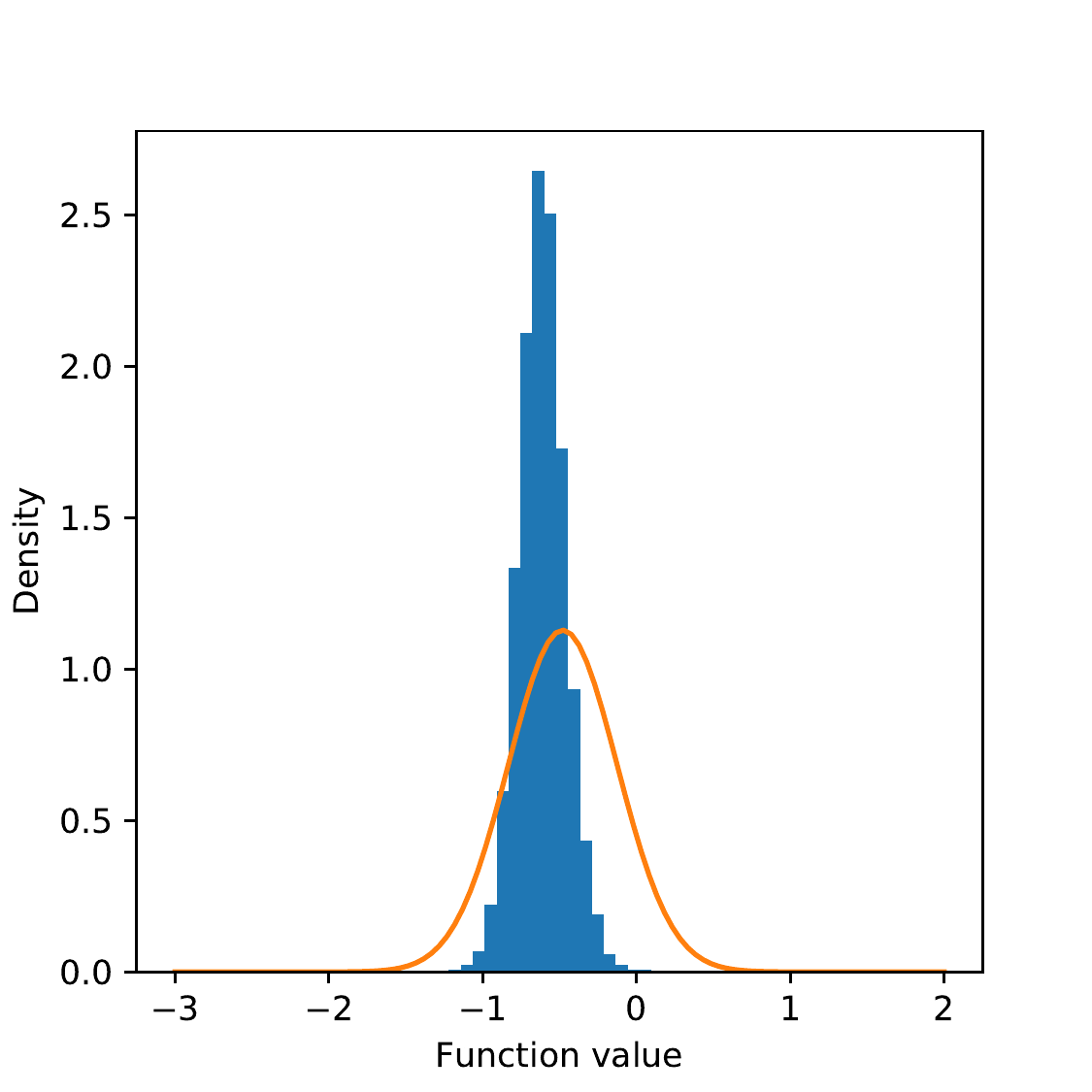}
\end{subfigure}
\end{center}
\caption{\label{fig:yacht} A comparison of the predictive distributions of a Bayesian deep network and a Gaussian process on the yacht hydrodynamics dataset. Left: the per-point log-densities of the two models. Right: predictive marginal distribution for the latent function on a randomly selected test point.}
\end{figure}

\subsection{Summary and discussion of Bayesian posterior comparison}

A summary of the datasets studied for comparing posteriors is given in Table \ref{table:datasets}. Of the datasets studied, the Bayesian neural network showed close agreement with the Gaussian process on five of the six datasets according to the various metrics used, the exception being the yacht dataset. It is notable that the yacht dataset has the highest dimensionality of those considered.

As already noted, our comparison method is computationally expensive as a result of the gold-standard MCMC algorithms used for Bayesian neural network inference. This means we are restricted to relatively small, low dimensional datasets. This caveat is particularly important in light of the yacht data results. On the other hand, we were also limited in the size of finite network we could consider for the same computational reason. As already discussed, the $50$ hidden unit networks we use are on the small end of the range of networks that have been studied in the literature \citep{Hernandez2015}, compared to values as high as $1200$ used in other works \citep{Blundell2015}. We would, of course, expect that where the model matches the assumption of our theory the agreement would become closer as the number of hidden units increases. As a result of the empirical analysis in Section \ref{section:mmd}, we would predict more difference if the number of hidden layers was substantially increased, though this has been relatively rare in the existing Bayesian literature thus far. 

Bringing these considerations together, it seems likely that some experiments in the literature studied under the banner of Bayesian deep learning would have given very similar results to a Gaussian process with the correct kernel. In the case where the two true posteriors are close, but the posterior approximation for the neural network is significantly worse than any approximation required for the Gaussian process, it would be expected that the Gaussian process would perform better. It should again be noted that the Bayesian neural network experiments were significantly slower than those conducted using the Gaussian process. The Snelson example took 44 hours on ten 3.2 GHz I7 CPU cores to obtain the two million samples required for the Bayesian neural network, where the Gaussian process took a matter of seconds. 

Practically, we suggest that the Bayesian deep learning community routinely compare their results to Gaussian processes with the kernels studied here. This will be facilitated by the release of our covariance function code built on GPflow \citep{Matthews2017}. Such a convention would significantly increase our empirical knowledge of the phenomenon studied in this paper.

\begin{center}
\begin{table}[t]
\small
\begin{tabular}{lcccc}
\toprule
{\bf Dataset}  & {\bf Training points} &{\bf Dimensionality} & {\bf Learnt hyperparams} &{\bf Figure} \\
\midrule
Small regression & 3 & 1 & \xmark & \ref{fig:one_dimensional_example} \\
Small XOR & 4 & 2 & \xmark & \ref{fig:bias_two_dimensional} \\
Random & 10 & 4 & \xmark & \ref{fig:bias_four} \\
Snelson & 100 & 1 & \cmark & \ref{fig:snelson} \\
Smooth XOR & 100 & 2 &\cmark & \ref{fig:smooth_xor} \\
Yacht & 100 & 6 & \cmark & \ref{fig:yacht} \\
\bottomrule
\end{tabular}
\caption{\label{table:datasets} Summary of datasets used for Bayesian posterior comparison. }
\end{table}
\end{center}

\section{Proof of the main theorem}\label{section:proof}

Let us first sketch the proof we will follow in this section. We first show that, for a countable set of inputs, the infinite-dimensional convergence problem can be reduced to a set of one-dimensional problems based on finite linear projections. When we examine these one-dimensional projections, we find their structure involves a sum of terms we refer to as \emph{summands}. For fixed width functions the summands are exchangeable, which leads us to consider central limit theorems for exchangeable arrays. A result of \citet{Blum1958} 
plays an essential role and requires certain moment conditions, that we show by induction through the layers of the network, starting nearest the input. There is a slight complication around the correct scaling of the summands to map onto the exchangeable central limit theorem, but this can be resolved with care.

We already pointed out in Section \ref{section:measureConverge} that with a countable index set convergence with respect to the metric $\sequenceMetric$ is equivalent to convergence of each finite-dimensional marginal. The Cram{\' e}r-Wold device \citep{Cramer1939} {\cite[p.~383]{Billingsey86}} states that convergence of a sequence of finite-dimensional vectors to some limit is equivalent to convergence on all possible linear projections to the corresponding real-valued random variable.  Putting these two results together we obtain the following lemma.

\begin{lemma}[Convergence of finite linear projections]\label{lemma:finiteLinear}
Consider a sequence of random functions $\stochasticProcess_{\sequenceIndex}$ taking values in $\sequenceSpace$ each defined on a countable input set $\countableIndexSet$, with the sequence of functions indexed by $\sequenceIndex$. Let $\projectionIndeces \subseteq \countableIndexSet$ be a finite subset of the input set. Further, let $\projectionIndex \in \realLine^{\projectionIndeces}$. Then convergence in distribution of the sequence of random functions $\stochasticProcess_{\sequenceIndex}$ taking values in $\sequenceSpace$ to a limiting random function $\limitStochasticProcess$ with respect to the metric $\sequenceMetric$ is equivalent to weak convergence of $\sum_{u \in \projectionIndeces} \stochasticProcess_{\sequenceIndex}(u) \projectionIndex_{u}$ to the corresponding finite linear projection $\sum_{u \in \projectionIndeces} \stochasticProcess_{*}(u) \projectionIndex_{u}$ for every such $\projectionIndeces$ and $\projectionIndex$.
\end{lemma}

Therefore our task is reduced to that of proving convergence of a sequence of real-valued random variables to another real-valued random variable -- a considerable simplification. In particular, we will leverage a theorem of \citet{Blum1958} on central limit theorems for exchangeable sequences. 

It will be convenient to consider a sequence of `infinite width, finite fan-out, networks'. By this we mean that the indices $\widthSymbol$ in the recursion \eqref{eq:recursion} can be thought of as running over all natural numbers instead of just up to $\width_{\depthSymbol}$ (hence infinite width). The limits of the sums in the recursion will retain the same finite values, which depend on the width functions evaluated at some $\couplingIndex$ (hence finite fan-out).  This makes only a superficial change because it adds extra copies of the same variables at each depth. For fixed $\couplingIndex$, these extra variables will not effect the downstream distribution of the network. The change is however useful in the book-keeping needed to prove convergence. We have defined a countable sequence of such networks because $\couplingIndex$ is a natural number. 

It will also be useful to slightly rewrite the defining initialisation and recursion \eqref{eq:recursion} from the more familiar form to one which is easier to manipulate:

\begin{align}\label{eq:initialRefined}
\indexedActivation{\smallestDepth}{\widthSymbol}{\inputSymbol} = \sum_{\widthSymbolB=1}^{\inputDimension} \indexedStandardNormal{\smallestDepth}{\widthSymbolPair} \inputSymbol_{\widthSymbolB}\sqrt{\weightVarianceScaled^{(\smallestDepth)}} + \indexedBias{\smallestDepth}{\widthSymbol} \, , \hspace{5pt} \widthSymbol \in \naturalNumbers \, ,
\end{align}

\noindent and:

\begin{align}\label{eq:recursionRefined}
\activity &= \nonlinearity( \indexedActivation{\depthSymbol}{\widthSymbol}{\inputSymbol} ) \, ,  \hspace{5 pt} \\
\indexedActivation{\depthSymbol+1}{\widthSymbol}{\inputSymbol}
 &= \frac{1}{\sqrt{\widthFunction_\depthSymbol(\sequenceVariable)}}\sum_{\widthSymbolB=1}^{\widthFunction_\depthSymbol(\sequenceVariable)} \indexedStandardNormal{\depthSymbol+1}{\widthSymbol, \widthSymbolB} \indexedActivity{\depthSymbol}{\widthSymbolB}{\inputSymbol}\sqrt{\weightVarianceScaled^{(\depthSymbol+1)}} + \indexedBias{\depthSymbol+1}{\widthSymbol}  \, , \hspace{5pt} \widthSymbol \in \naturalNumbers \, ,
\end{align}

\noindent where:

\begin{align}
\indexedStandardNormal{\depthSymbol}{\widthSymbol,\widthSymbolB} \sim \normal(0,1) \hspace{7 pt} \text{i.i.d} \hspace{3 pt} \forall \depthSymbol, \widthSymbol, \widthSymbolB \, .
\end{align}

\noindent This amounts to reparameterising the weights in terms of standard normals and making the previously mentioned infinite extension of the width variable $\widthSymbol$. We re-emphasize that neither step changes the distribution over final function values. With the aim of mapping onto Lemma \ref{lemma:finiteLinear} we make the following definitions:

\begin{definition}[Projections and summands] The projections are defined in terms of a finite linear projection of the function values without biases:

\begin{align}
\projection^{(\depthSymbol)}(\projectionIndeces,\projectionCoefficients)\atWidth = \sum_{(\datapoint,\widthSymbol) \in \projectionIndeces} \projectionCoefficients^{(\datapoint,\widthSymbol)}\left[\indexedActivation{\depthSymbol}{\widthSymbol}{\inputSymbol}\atWidth - \indexedBias{\depthSymbol}{\widthSymbol}\right]\equationFullStop
\end{align}

\noindent where $\projectionIndeces \subset \indexSet \times \naturalNumbers$ is a finite set of tuples of data points and indices of pre-nonlinearities, with $\indexSet = (\inputSymbol[i])_{i=1}^\infty$. $\projectionCoefficients \in \realLine^{|\projectionIndeces|}$ is a vector parameterising the linear projection. The suffix $\atWidth$ indicates that the corresponding width functions are instantiated with input $\couplingIndex$.

The summands are defined as:

\begin{align}\summand^{(\depthSymbol)}_{\widthSymbolB}(\projectionIndeces,\projectionCoefficients)\atWidth := \sum_{(\datapoint,\widthSymbol) \in \projectionIndeces} \projectionCoefficients^{(\datapoint,\widthSymbol)} \indexedStandardNormal{\depthSymbol}{\widthSymbol, \widthSymbolB} \indexedActivity{\depthSymbol-1}{\widthSymbolB}{\inputSymbol}\atWidth\sqrt{\weightVarianceScaled^{(\depthSymbol)}} \, ,
\end{align}
 
\noindent in order to ensure the summation relation

\begin{align}\label{eq:projection_as_array}
\projection^{(\depthSymbol)}(\projectionIndeces,\projectionCoefficients)\atWidth := \frac{1}{\sqrt{\widthFunction_{\depthSymbol-1}(n)}}\sum_{\widthSymbolB=1}^{\widthFunction_{\depthSymbol-1}(n)}\summand^{(\depthSymbol)}_{\widthSymbolB}(\projectionIndeces,\projectionCoefficients)\atWidth\, .
\end{align}

\end{definition}

The last relation follows from applying the definitions and re-arranging the order of summation. Note the similarity between the definition of projections used here and in Lemma \ref{lemma:finiteLinear}.  We next show that the summands are exchangeable.

\begin{lemma}[Exchangeability of summands]\label{lemma:exchange}
For each fixed $\sequenceVariable$ and $\depthSymbol \in \{2,\ldots,D+1\}$, the countable sequence of summands $\summand^{(\depthSymbol)}_{\widthSymbolB}(\projectionIndeces,\projectionCoefficients)\atWidth$ are an exchangeable sequence with respect to the index $\widthSymbolB$.
\end{lemma}

\begin{proof}
To prove the lemma we use de Finetti's theorem, which states that a sequence of random variables is exchangeable if and only if they are i.i.d. conditional on some set of random variables. It is therefore sufficient to exhibit such as set of random variables. To do this we apply the recursion. Removing some multiplicative constants we have:
\begin{align}
\summand^{(\depthSymbol)}_{\widthSymbolB}(\projectionIndeces,\projectionCoefficients)\atWidth &\propto \sum_{(\datapoint,\widthSymbol) \in \projectionIndeces} \projectionCoefficients^{(\datapoint,\widthSymbol)} \indexedStandardNormal{\depthSymbol}{\widthSymbol, \widthSymbolB} \indexedActivity{\depthSymbol-1}{\widthSymbolB}{\inputSymbol}\atWidth \\
&=\sum_{(\datapoint,\widthSymbol) \in \projectionIndeces} \projectionCoefficients^{(\datapoint,\widthSymbol)} \indexedStandardNormal{\depthSymbol}{\widthSymbol, \widthSymbolB} \nonlinearity\left( \frac{1}{\sqrt{\widthFunction_{\depthSymbol-2}(n)}}\sum_{\widthSymbolB=1}^{\widthFunction_{\depthSymbol-2}(n)}\indexedStandardNormal{\depthSymbol-1}{\widthSymbolB,\widthSymbolC} \indexedActivity{\depthSymbol-2}{\widthSymbolC}{\inputSymbol}\atWidth \sqrt{\weightVarianceScaled^{(\depthSymbol-1)}} + \indexedBias{\depthSymbol-1}{\widthSymbolB} \right) \, ,
\end{align}

\noindent with the convention that $h_0(n) = M$ and $g^{(0)}_k(x) = x_k$ for $k=1,\ldots,M$. Conditional on the finite set of random variables $\left \lbrace \indexedActivity{\depthSymbol-2}{\widthSymbolC}{\inputSymbol}\atWidth : \widthSymbolC=1,...,\width_{\depthSymbol-2}, \inputSymbol \in \projectionIndeces_{\indexSet} \right\rbrace$ (where $\projectionIndeces_{\indexSet}$ is the set of inputs points in $\projectionIndeces$), the summands are independent and identically distributed. 

\end{proof}

Thus we are led to consider central limit theorems for sequences of exchangeable sequences. The work of \citet{Blum1958} will provide our starting point.

\begin{theorem}[CLT for exchangeable sequences \citep{Blum1958}]\label{theorem:blum}
For each positive integer $\rowIndex$ let $\left( \genericRV_{\rowIndex, \colIndex} ; \colIndex = 1,2,... \right)$ be an infinitely exchangeable process with mean zero, variance one, and finite absolute third moment. Define

\begin{equation}
\generalSum_{\rowIndex} = \frac{1}{\sqrt{\rowIndex}}\sum_{\colIndex=1}^{\rowIndex} \genericRV_{\rowIndex,\colIndex} \equationFullStop
\end{equation}

\noindent Then if the following conditions hold:

\begin{enumerate}
\item $\expectation{\genericRV_{\rowIndex, 1} \genericRV_{\rowIndex, 2}}{\rowIndex} = \littleO(\frac{1}{\rowIndex}) $
\item $ \lim_{\rowIndex \to \infty }\expectation{\genericRV_{\rowIndex, 1}^{2} \genericRV_{\rowIndex, 2}^{2}}{\rowIndex} = 1 $
\item $ \expectation{|\genericRV_{\rowIndex, 1}|^{3}}{\rowIndex} = \littleO(\sqrt{\rowIndex}) $
\end{enumerate}

Then $\generalSum_{\rowIndex}$ converges in distribution to a standard normal. 

\end{theorem}

This is effectively a generalisation of the classical CLT from independent identically distributed variables to the more general class of exchangeable ones. We will need to address the fact that the theorem applies to unit variance variables and that we have non-identity width functions. The next lemma adapts the work of Blum et al. to our specific requirements.

\begin{lemma}[Adapted CLT for sequences of exchangeable sequences]\label{lemma:gcltb}
For each positive integer $\rowIndex$ let $\left( \genericRV_{\rowIndex, \colIndex} ; \colIndex = 1,2,... \right)$ be an infinitely exchangeable process with mean zero, finite variance $\rowVariance$, and finite absolute third moment.  Suppose also that the variance has a limit $\lim_{\rowIndex \to \infty}\rowVariance = \limitVariance$. Define

\begin{equation}
\generalSum_{\rowIndex} = \frac{1}{\sqrt{\monotoneCLTfunc(\rowIndex)}}\sum_{\colIndex=1}^{\monotoneCLTfunc(\rowIndex)} \genericRV_{\rowIndex,\colIndex} \, , 
\end{equation}

\noindent where $\monotoneCLTfunc : \naturalNumbers \mapsto \naturalNumbers$ is a strictly increasing function. Then if the following conditions hold:

\begin{enumerate}[a)]
\item $\expectation{\genericRV_{\rowIndex, 1} \genericRV_{\rowIndex, 2}}{\rowIndex} = 0 $
\item $ \lim_{\rowIndex \to \infty }\expectation{\genericRV_{\rowIndex, 1}^{2} \genericRV_{\rowIndex, 2}^{2}}{\rowIndex} = \limitStd^{4} $
\item $ \expectation{|\genericRV_{\rowIndex, 1}|^{3}}{\rowIndex} = \littleO(\sqrt{\monotoneCLTfunc(\rowIndex)}) $
\end{enumerate}

Then $\generalSum_{\rowIndex}$ converges in distribution to $\normal(0,\limitVariance)$, where $\normal(0,0)$ is interpreted as converging to $0$. 

\end{lemma}

We postpone the proof of Lemma \ref{lemma:gcltb} until Appendix \ref{section:CLTproof}. Our next step will be to apply Lemma \ref{lemma:gcltb} to the projections and summands by showing they meet each condition. We first establish the existence of a limiting variance. 

\begin{lemma}[Limiting variance]\label{lemma:limitVar}
The limiting variance, defined as

\begin{align}
\limitVar := \lim_{\sequenceVariable \to \infty}\projectionVar \, ,
\end{align}

\noindent exists, where $\projectionVar$ is the variance of the random variables $\summand^{(\depthSymbol)}_{\widthSymbolB}(\projectionIndeces,\projectionCoefficients)\atWidth$, and has the value

\begin{align}
\limitVar = \alpha^{T} \genericCov(\projectionIndeces) \alpha \, ,
\end{align}

\noindent where $\genericCov \in \realLine^{\projectionIndeces \times \projectionIndeces}$ is the Gram matrix implied by the recursion \ref{lemma:recursion} without a bias correction on the final layer.

\end{lemma}

\noindent The proof of this Lemma can be found in Appendix \ref{section:remainingLemmas}.

\begin{lemma}[Convergence in distribution of projections]\label{lemma:standardProjection}
As $\sequenceVariable \to \infty$ the projection $\projection^{(\depthSymbol)}(\projectionIndeces,\projectionCoefficients)\atWidth$ converges in distribution to $\normal(0,\limitVar)$. 
\end{lemma}

The full details of Lemma \ref{lemma:standardProjection} are explained in Appendix \ref{section:remainingLemmas}. Here we outline the key points of the approach. We apply Lemma \ref{lemma:gcltb} to the projections, using the fact that the summands are exchangeable for each $\sequenceVariable$ and with the limiting variance $\limitVar$ derived in Lemma \ref{lemma:limitVar}. Condition a) of Lemma \ref{lemma:gcltb} follows straightforwardly from the fact that the summands are uncorrelated. That Condition c) is fulfilled is intuitively reasonable given that we in fact expect this absolute third moment to tend to a constant. Condition c) will however still need to be shown carefully. This leaves Condition b). Convergence of the expectation of a sequence of random variables can be ensured if the sequence is uniformly integrable and the sequence converges in distribution \citep{billingsey68}. Thus the main work of Appendix \ref{section:remainingLemmas} is to prove these conditions in our case, by induction forwards through the network.

Lemma \ref{lemma:standardProjection} shows consistency of convergence of the finite linear projections of the pre-bias function distribution with the stated Gaussian process. By Lemma \ref{lemma:finiteLinear}, this is sufficient for convergence in distribution to the Gaussian process. As the biases are normally distributed it is straightforward to add them and get the final result. Therefore we are done.

\section{Desirability of Gaussian process behaviour and methods to avoid it}\label{section:avoidingGPs}

When using deep Bayesian neural networks as priors, the emergence of Gaussian priors raises important questions in the cases where it is applicable, even if one sets aside questions of computational tractability. The kernels considered in this paper have not been commonly used in the Gaussian process literature and warrant further analysis. It has been argued by previous authors that there are important cases where kernel machines with \emph{local} kernels will perform badly \citep{Bengio2005}. The analysis applies to the posterior mean of a Gaussian process. The kernels considered in this paper do not meet the strict definition of what could be considered local, though the Euclidean inner product between two points is sufficient to compute the corresponding covariance. In any case, the fact remains that a Gaussian process with a fixed kernel does not use a learnt hierarchical representation. Such representations are widely regarded to be essential to the success of deep learning. A complication to consider is when a hierarchical treatment of the model is taken, learning model hyperparameters. Typically only a few such hyperparameters are used and it seems unlikely this could offer the same benefits as full representation learning. Using significantly more hyperparameters would move the model beyond the scope of this paper. \citet[p.~547]{MacKay2002} famously reflected on what is lost when taking the Gaussian process limit of a single hidden layer network, remarking that Gaussian processes will not learn hidden features. \citet[p.~43]{Neal1996} makes similar comments and also expresses the hope that Bayesian neural networks could expand the range of probabilistic models beyond Gaussian processes. In light of the results in this paper for networks with more than one hidden layer these considerations are of considerable importance going forward.

There is literature on learning the representation of a standard, usually structured, network composed with a Gaussian process \citep{Wilson16,Wilson16b,AlShedivat17}. This differs from the assumed paradigm of this paper, where all model complexity is specified probabilistically and we do not assume convolutional, recurrent or other problem specific structure. 

Within the paradigm considered here, the question therefore arises as to what can be done to avoid marginal Gaussian process behaviour if it is not desired. Speaking loosely, to stop the onset of the central limit theorem and the approximate analogues discussed in this paper one needs to make sure that one or more of its conditions is far from being met. Since the chief conditions on the summands are independence, bounded variance and many terms, violating these assumptions will remove Gaussian process behaviour. Deep Gaussian processes \citep{Damianou2013} are not close to standard Gaussian processes marginally because they are typically used with narrow intermediate layers. It can be challenging to choose the precise nature of these narrow layers a priori. \citet{Neal1996} suggests using networks with infinite variance in the activities. With a single hidden layer and correctly scaled, these networks become alpha stable processes in the wide limit. Neal also discusses variants that destroy independence by coupling weights. These alternatives each arguably have a mechanism to discover hierarchies of features. Again, given the convergence results for multiple hidden layer networks from this paper, there is now further motivation to study the non-Gaussian alternatives as well.

\section{Conclusions}

Studying the limiting behaviour of distributions on feedforward neural networks has been a fruitful avenue for understanding these models historically. In this paper we have formalised and extended prior results by \citet{Neal1996} to deep networks. In particular, we have shown that, under broad conditions, as we make the architecture increasingly wide, the implied random function converges in distribution to a Gaussian process. Our empirical study using MMD suggests that this behaviour is exhibited in a variety of models of size comparable to networks used in the literature. This led us to juxtapose finite Bayesian neural networks with their Gaussian process analogues. In several cases there was close agreement, leading us to conclude that it is likely some results from the existing Bayesian deep learning literature would be very similar to those obtained with the corresponding Gaussian process model. We recommend that empirical investigation of Bayesian neural networks should routinely include comparison to their Gaussian process analogue. If Gaussian process behaviour is desired then exact and approximate inference using the analytic properties of Gaussian processes should be considered as an alternative to neural network inference. Since Gaussian processes have an equivalent flat representation then in the context of deep learning there may well be cases where the behaviour is not desired and steps should be taken to avoid it.

We view these results as a new opportunity to further the understanding of neural networks in the work that follows. Initialisation and learning dynamics are crucial topics of study in modern deep learning which require that we understand random networks. Bayesian neural networks should offer a principled approach to generalisation but this relies on successfully approximating a clearly understood prior. In illustrating the continued importance of Gaussian processes as limit distributions, we hope that our results will further research in these broader areas.

\section{Acknowledgements}

We wish to thank Neil Lawrence and Jascha Sohl-Dickstein for helpful conversations. We also thank anonymous reviewers of previous versions for their insights. Alexander Matthews and Zoubin Ghahramani acknowledge the support of EPSRC Grant EP/N014162/1 and EPSRC Grant EP/N510129/1 (The Alan Turing Institute). Jiri Hron holds a Nokia CASE Studentship. Mark Rowland acknowledges support by EPSRC grant EP/L016516/1 for the Cambridge Centre for Analysis. Richard E. Turner is supported by Google as well as EPSRC grants EP/M0269571 and EP/L000776/1.

\bibliography{bibliography}

\appendix

\section{Adapting the exchangeable CLT of Blum et al. 1958.}\label{section:CLTproof}

This section gives further detail on our adaption of Theorem \ref{theorem:blum} to our specific needs. It states and proves an intermediate Lemma \ref{lemma:gclta} and then using that lemma gives the postponed proof of Lemma \ref{lemma:gcltb}.

\begin{lemma}[Variance adapted CLT for sequences of exchangeable sequences]\label{lemma:gclta}
For each positive integer $\rowIndex$ let $\left( \genericRV_{\rowIndex, \colIndex} ; \colIndex = 1,2,... \right)$ be an infinitely exchangeable process with mean zero, finite variance $\rowVariance$, and finite absolute third moment.  Suppose also that the variance has a limit $\lim_{\rowIndex \to \infty}\rowVariance = \limitVariance$. Define

\begin{equation}
\generalSum_{\rowIndex} = \frac{1}{\sqrt{\rowIndex}}\sum_{\colIndex=1}^{\rowIndex} \genericRV_{\rowIndex,\colIndex} \equationFullStop
\end{equation}

\noindent Then if the following conditions hold:

\begin{enumerate}[i.]
\item $\expectation{\genericRV_{\rowIndex, 1} \genericRV_{\rowIndex, 2}}{\rowIndex} = 0 $
\item $ \lim_{\rowIndex \to \infty }\expectation{\genericRV_{\rowIndex, 1}^{2} \genericRV_{\rowIndex, 2}^{2}}{\rowIndex} = \limitStd^{4} $
\item $ \expectation{|\genericRV_{\rowIndex, 1}|^{3}}{\rowIndex} = \littleO(\sqrt{\rowIndex}) $
\end{enumerate}

Then $\generalSum_{\rowIndex}$ converges in distribution to $\normal(0,\limitVariance)$, where $\normal(0,0)$ is interpreted as converging to $0$.

\end{lemma}

\begin{proof}[Proof of Lemma \ref{lemma:gcltb}] Either $\limitVariance=0$ or it does not. We deal with each case separately. 

In the case where  $\limitVariance=0$, we have:

\begin{align}
\Var[\generalSum_{\rowIndex}] &= \frac{1}{\rowIndex}\Var\left[\sum_{\colIndex=1}^{\rowIndex}\genericRV_{\rowIndex,\colIndex} \right] \\
&= \frac{1}{\rowIndex}\left( \sum_{\colIndex=1}^{\rowIndex} \Var[\genericRV_{\rowIndex,\colIndex}] + \sum_{\colIndex \neq \colIndex'}\Cov[\genericRV_{\rowIndex,\colIndex},\genericRV_{\rowIndex,\colIndex'}] \right) \\
&= \Var[\genericRV_{\rowIndex,1}] = \rowVariance,
\end{align}

\noindent where we have used property (i) that the distinct elements in a row are uncorrelated. Now the proof can take a similar route to that used in proving the weak law of large numbers with finite variance. Chebyshev's inequality we have:

\begin{align}
\Prob( |\generalSum_\rowIndex| \leq \beta ) \leq \frac{\rowVariance}{\beta^2} \, ,
\end{align}

\noindent for all $\beta>0$, So that:

\begin{align}
\Prob( |\generalSum_\rowIndex| > \beta ) \leq  1-  \frac{\rowVariance}{\beta^2} \, .
\end{align}

That is to say $\generalSum_{\rowIndex}$ converges in probability to $0$. In such a case of a constant target convergence in probability is equivalent to convergence in distribution.

In the case where $\limitVariance\neq 0$ there is always some $M$ such that for $\rowIndex \geq M$ $\rowVariance > 0$ by the definition of a limit. Let us assume we are in this range of $\rowIndex$. Then the standardised values $\frac{\genericRV_{\rowIndex,\colIndex}}{\rowStd}$ will obey the conditions of Theorem \ref{theorem:blum}, which we now show term by term. We clearly have mean zero, unit variance and finite third moment, and conditions 1) and i) are identical. This leaves us to validate conditions 2) and 3). Starting from ii) we have:

\begin{align}
\lim_{\rowIndex \to \infty }\expectation{\genericRV_{\rowIndex 1}^{2} \genericRV_{\rowIndex 2}^{2}}{\rowIndex} = \limitStd^{4} \\
\lim_{\rowIndex \to \infty }\expectation{\genericRV_{\rowIndex 1}^{2} \genericRV_{\rowIndex 2}^{2}}{\rowIndex} \lim_{\rowIndex \to \infty } \expectation{\frac{1}{\rowStd^{4}}}{\rowIndex} = 1 \\
\lim_{\rowIndex \to \infty }\expectation{\frac{\genericRV_{\rowIndex 1}^{2} \genericRV_{\rowIndex 2}^{2}}{\rowStd^{4}}}{\rowIndex} = 1
\end{align}

\noindent which clearly implies condition 2). Starting from condition iii) we have:

\begin{align}
\lim_{\rowIndex \to \infty }\left[ \frac{|\genericRV_{\rowIndex 1}|^{3}}{\sqrt{\rowIndex}}\right] &= 0 \\
\lim_{\rowIndex \to \infty }\left[ \frac{|\genericRV_{\rowIndex 1}|^{3}}{\rowStd^3\sqrt{\rowIndex}}\right] &= \frac{1}{\limitStd^{3}}\lim_{\rowIndex \to \infty }\left[ \frac{|\genericRV_{\rowIndex 1}|^{3}}{\sqrt{\rowIndex}}\right] = 0
\end{align}

\noindent which implies condition 3). Since multiplication by a constant is a continuous function we have therefore showed that $\generalSum_{\rowIndex} \frac{\limitStd}{\rowStd}$ converges in distribution to $\normal(0,\limitVar)$. 

Note that the sequence $|\generalSum_{\rowIndex} \frac{\limitStd}{\rowStd} - \generalSum_{\rowIndex}|$ converges \emph{surely} to $0$. This certainly implies convergence in probability of the same sequence to zero. We can therefore invoke a general result on convergence of sequences that says if a sequence of random variables $X_i$ converges to $X_*$ and $|X_i - Y_i|$ converges in probability to zero, then $Y_i$ converges in distribution to $X_*$ \citep{vaart1998}. 

\end{proof}

\begin{proof}[Proof of Lemma \ref{lemma:gcltb}] Lemma \ref{lemma:gclta} applies to what are known as triangular arrays in the literature. This lemma is the generalisation to arrays that are not strictly triangular. To do this we embed the non-triangular array in a large triangular one. We fill the extra spaces with standard normal random variables. This gives an interleaved sequence. The terms we actually care about will obey the necessary conditions if conditions 1) 2) and 3) if a) b) and c) hold. The conditions 1) 2) and 3) will hold trivially for the standard normal rows. Thus the whole sequence converges in distribution. But since any subsequence also converges in distribution we get our required result. \end{proof}

\section{Details of the proof of Theorem \ref{thm:main}}

Here, we summarise the high-level structure of the proof of Theorem \ref{thm:main}. The argument is inductive, showing sequentially that the hidden units in each layer of the network converge in distribution; to avoid repetition, all mentions of convergence in distribution of infinite-dimensional random variables in what follows are specifically with respect to the topology generated by the metric $\rho$ introduced in Section \ref{section:measureConverge}. The main part of the inductive argument is summarised in the following proposition.

\begin{proposition}\label{prop:inductive}
	For any $\depthSymbol \in \{2, \ldots, \depthSymbol + 1\}$, suppose that the collection of random variables $\{f^{(\depthSymbol-1)}_i(\inputSymbol)\atWidth\}_{\widthSymbol \in \naturalNumbers, \inputSymbol \in \indexSet}$ converges in distribution as $n \rightarrow \infty$ to a centred Gaussian with covariance function of the form given in Lemma \ref{lemma:recursion}. Then any finite linear combination $\projection^{(\depthSymbol)}(\projectionIndeces,\projectionCoefficients)\atWidth$ (with $\projectionIndeces \subset \indexSet \times \naturalNumbers$ finite and $\projectionCoefficients \in \realLine^{\projectionIndeces}$) of pre-nonlinearities at the next layer also converges in distribution to a centred Gaussian of the form described in Lemma \ref{lemma:limitVar}.
\end{proposition}

Note that the conclusion of Proposition \ref{prop:inductive} leads to the statement of Lemma \ref{lemma:standardProjection}.
By the Cram\'er-Wold device discussed in Section \ref{section:proof}, the convergence of the finite linear projections established in Proposition \ref{prop:inductive} guarantees convergence of all finite-dimensional marginal distributions. Adding in the independent bias terms yields convergence of finite-dimensional marginals of the pre-activations at layer $\mu$; this may be demonstrated via a standard argument using characteristic functions. Due to the remarks on weak convergence in Section \ref{section:measureConverge}, convergence in distribution of all finite-dimensional marginals guarantees convergence in distribution of the full collection of random variables $\{f^{(\depthSymbol)}_i(\inputSymbol)\atWidth\}_{\widthSymbol \in \naturalNumbers, \inputSymbol \in \indexSet}$ in the next layer, completing the inductive step.

The proof of Theorem \ref{thm:main} is then concluded by observing that the pre-nonlinearities in the first hidden layer, $\{f^{(1)}_i(\inputSymbol)\atWidth\}_{\widthSymbol \in \naturalNumbers, \inputSymbol \in \indexSet}$, have a fixed Gaussian distribution that does not depend on $n$.

We thus turn our attention to proving Proposition \ref{prop:inductive}. The main idea is to use Lemma~\ref{lemma:gcltb}, taking each of the random variables $X_{n, i}$ (for $i \in \mathbb{N}, n \in \mathbb{N}$) appearing in the statement of the Lemma to be the summands appearing in the finite linear projections $\projection^{(\depthSymbol)}(\projectionIndeces,\projectionCoefficients)\atWidth$:
\begin{align}
X_{n, i} = \summand^{(\depthSymbol)}_{\widthSymbol}(\projectionIndeces,\projectionCoefficients)\atWidth \, .
\end{align}
Addressing the conditions of Lemma \ref{lemma:gcltb}, we note that the exchangeability condition is provided by Lemma \ref{lemma:exchange}, the mean-zero condition is immediate, the limiting variance condition is dealt with by Lemma~\ref{lemma:limitVar}. 
Condition a) of Lemma~\ref{lemma:gcltb} holds trivially as the random variables $X_{n1}$ and $X_{n2}$ are mean-zero and uncorrelated. The remaining conditions of Lemma \ref{lemma:gcltb} are dealt with through the following results; Lemma~\ref{lemma:convg_product_2nd_mmnts} deals with Condition b), whilst Lemma~\ref{lemma:3rd_abs_mmnt_summands} deals with the~growth of third absolute moments as required by Condition c).


\begin{lemma}[Convergence of $\expectation{|\genericRV_{\rowIndex, 1} \genericRV_{\rowIndex, 2}|^{2}}{}$]\label{lemma:convg_product_2nd_mmnts}
Consider arbitrary $\depthSymbol \in \{2, \ldots, \depth + 1\}$ and the~corresponding set of random variables $\{\indexedActivation{\depthSymbol}{\widthSymbol}{\inputSymbol}\atWidth\}_{(\widthSymbol, \inputSymbol) \in \projectionIndeces}$. Assume that the~countably infinite vector of random variables $\{ \indexedActivation{\depthSymbol-1}{\widthSymbol}{\inputSymbol}\atWidth \}_{\widthSymbol \in \naturalNumbers, \inputSymbol \in \indexSet}$ converges in distribution to a~centred Gaussian process with covariance specified by the~recursion in Lemma~\ref{lemma:recursion} as $\sequenceVariable \to \infty$. Then
$$\lim_{\sequenceVariable \to \infty} \expectation{|\genericRV_{\sequenceVariable, 1} \genericRV_{\sequenceVariable, 2}|^2}{} = \limitStd^{4} \, .$$
\end{lemma}

\begin{lemma}[Bound on $\expectation{|\genericRV_{\rowIndex, 1}|^{3}}{}$]\label{lemma:3rd_abs_mmnt_summands}
For arbitrary given $\projectionCoefficients$, $\projectionIndeces$, and $\depthSymbol \in \{1, 2, \ldots, \depth+1\}$, $\expectation{|\genericRV_{\rowIndex, 1}|^{3}}{} < c < \infty$ with $c$ independent of $\sequenceVariable$. Thus $\expectation{|\genericRV_{\rowIndex, 1}|^{3}}{} = \littleO(\sqrt{\monotoneCLTfunc(\rowIndex)})$.
\end{lemma}

Thus, all that remains to establish Theorem \ref{thm:main} is the proof of these intermediate lemmas; the proofs are given in the sections that follow.

%
%
%


\subsection{Proofs of main lemmas and corollaries}\label{section:remainingLemmas}

Throughout this section, we simplify the~notation by defining
\begin{align*}
\summand^{(\depthSymbol)}_{j}(\projectionIndeces,\projectionCoefficients)\atWidth
&:= 
\projectionCoefficients^\top \tilde{\activitySymbol}_j^{(\depthSymbol)}\atWidth
&& j \in \naturalNumbers \, ,
\nonumber \\
\tilde{\activitySymbol}_j^{(\depthSymbol)}\atWidth_i 
&:= 
\indexedStandardNormal{\depthSymbol}{(i), j} \indexedActivity{\depthSymbol-1}{j}{\inputSymbol_{(i)}}\atWidth
&& i \in \{1, \ldots, |\projectionIndeces|\}
\, ,
\end{align*}
where $((i), \inputSymbol_{(i)})$ is the~$i$\textsuperscript{th} member of the~set $\projectionIndeces$. Without loss of generality, in what follows we will take $\weightVarianceDepthScaled = 1$ to lighten notation.

To prove Lemma~\ref{lemma:convg_product_2nd_mmnts}, we need to know the~value of $\limitStd^4$ where $\limitStd^2 = \limitVar$ as defined in Lemma~\ref{lemma:limitVar}. Lemma~\ref{lemma:conditional_limit_std} combined with the~inductive propagation of convergence in distribution verifies Lemma~\ref{lemma:limitVar} and thus yields $\limitStd^4$.

\begin{lemma}\label{lemma:conditional_limit_std}
Consider arbitrary $\depthSymbol \in \{2, \ldots, \depth+1\}$. Assume that the~countably infinite vector of random variables $\{ \indexedActivation{\depthSymbol-1}{\widthSymbol}{\inputSymbol}\atWidth \}_{\widthSymbol \in \naturalNumbers, \inputSymbol \in \indexSet}$ converges in distribution to a~centred Gaussian process with covariance specified by the~recursion in Lemma~\ref{lemma:recursion} as $\sequenceVariable \to \infty$. Then
\begin{equation*}
\limitVar = \lim_{\sequenceVariable \to \infty}\projectionVar = \alpha^{T} \genericCov(\projectionIndeces) \alpha \, .
\end{equation*}
\end{lemma}

\begin{proof} 
Lemma~\ref{lemma:limitVar} introduces $\genericCov(\projectionIndeces)$ which is the~marginal covariance of the~limiting Gaussian process without the~bias term (c.f.\ the~recursion in Lemma~\ref{lemma:recursion}).

We use exchangeability of $\summand^{(\depthSymbol)}_{j}(\projectionIndeces,\projectionCoefficients)\atWidth$ over the~index $j$ to obtain
\begin{equation*}
\projectionVar
=
\expectation{
	(\summand^{(\depthSymbol)}_{1}(\projectionIndeces,\projectionCoefficients)\atWidth)^2
}{}
= 
\projectionCoefficients^\top 
\expectation{
	\tilde{\activitySymbol}_1^{(\depthSymbol)}\atWidth
	\tilde{\activitySymbol}_1^{(\depthSymbol)}\atWidth^\top
}{}
\projectionCoefficients
\, .
\end{equation*}
Hence the~limit of $\projectionVar$ is fully determined by the~behaviour of $\tilde{\activitySymbol}_1^{(\depthSymbol)}\atWidth$ as $\sequenceVariable \to \infty$.

We can thus focus on individual entries of the~expectation on the~RHS of the~above equation. For entry $(i, j)$ with $i, j \in \{1, \ldots, |\projectionIndeces|\}$, we have
\begin{align*}
\expectation{
	\tilde{\activitySymbol}_1^{(\depthSymbol)}\atWidth_i \,
	\tilde{\activitySymbol}_1^{(\depthSymbol)}\atWidth_j
}{}
=
\delta_{(i) = (j)}
\expectation{
	\indexedActivity{\depthSymbol-1}{1}{\inputSymbol_{(i)}}\atWidth
	\indexedActivity{\depthSymbol-1}{1}{\inputSymbol_{(j)}}\atWidth
}{}
\, .
\end{align*}

Since $\indexedActivity{\depthSymbol-1}{1}{\inputSymbol_{(k)}}\atWidth = \nonlinearity(\indexedActivation{\depthSymbol-1}{1}{\inputSymbol_{(k)}}\atWidth), k \in \{i, j\}$ and the~collection  $\{\indexedActivation{\depthSymbol-1}{1}{\inputSymbol}\atWidth\}_{\inputSymbol \in \indexSet}$ converges in distribution as $\sequenceVariable \to \infty$ by assumption, we can use the~continuity of $\nonlinearity$ and the~continuous mapping theorem to deduce that the~post-nonlinearities are converging in distribution. Because the~function $h(x_1, x_2) = x_1 x_2$ is continuous, we can apply the~continuous mapping theorem again to deduce that the~two-way products of post-nonlinearities are converging in distribution to the~limit specified by the~pushforward of the~limiting multivariate normal distribution.

Theorem~3.5 in \citep{billingsey68} tells us that the~expectation
\begin{equation*}
\lim_{\sequenceVariable \to \infty}
\expectation{
	\indexedActivity{\depthSymbol-1}{1}{\inputSymbol_{(i)}}\atWidth
	\indexedActivity{\depthSymbol-1}{1}{\inputSymbol_{(j)}}\atWidth
}{}
=
\expectation{
	\indexedActivity{\depthSymbol-1}{1}{\inputSymbol_{(i)}}\atLimit
	\indexedActivity{\depthSymbol-1}{1}{\inputSymbol_{(j)}}\atLimit
}{}
\, ,
\end{equation*}
if the~family of random variables indexed by $\sequenceVariable$ is uniformly integrable. Uniform integrability is a~corollary of Lemma~\ref{lemma:integrability_postnonlin}. Inspection of the~recursion in Lemma~\ref{lemma:limitVar} finishes the~proof.
\end{proof}


\begin{proof}[Proof of Lemma~\ref{lemma:convg_product_2nd_mmnts}]
Substituting for $\genericRV_{\sequenceVariable, 1}$ and $\genericRV_{\sequenceVariable, 2}$, we have
\begin{equation}\label{eq:2nd_mmnt_product_summands}
\expectation{
	\left|
		\summand^{(\depthSymbol)}_{1}(\projectionIndeces,\projectionCoefficients)\atWidth
		\summand^{(\depthSymbol)}_{2}(\projectionIndeces,\projectionCoefficients)\atWidth
	\right|^2
}{}
=
\projectionCoefficients^\top
\expectation{
	\tilde{\activitySymbol}_1^{(\depthSymbol)}\atWidth 
	\tilde{\activitySymbol}_1^{(\depthSymbol)}\atWidth^\top 
	\projectionCoefficients \projectionCoefficients^\top
	\tilde{\activitySymbol}_2^{(\depthSymbol)}\atWidth 
	\tilde{\activitySymbol}_2^{(\depthSymbol)}\atWidth^\top 
}{}
\projectionCoefficients
\, .
\end{equation}
The~expectation on the~RHS can be rewritten as
\begin{equation*}
\expectation{
	\tilde{\activitySymbol}_1^{(\depthSymbol)}\atWidth 
	\tilde{\activitySymbol}_1^{(\depthSymbol)}\atWidth^\top 
	\projectionCoefficients \projectionCoefficients^\top
	\tilde{\activitySymbol}_2^{(\depthSymbol)}\atWidth 
	\tilde{\activitySymbol}_2^{(\depthSymbol)}\atWidth^\top 
}{}
=
\sum_{i=1}^{|\projectionIndeces|} \sum_{j=1}^{|\projectionIndeces|}
\projectionCoefficients_i \projectionCoefficients_j
\expectation{
	\tilde{\activitySymbol}_1^{(\depthSymbol)}\atWidth_i \,
	\tilde{\activitySymbol}_2^{(\depthSymbol)}\atWidth_j \,
	\tilde{\activitySymbol}_1^{(\depthSymbol)}\atWidth \,
	\tilde{\activitySymbol}_2^{(\depthSymbol)}\atWidth^\top
}{}
\, ,
\end{equation*}
Hence the~limit of the~LHS of Equation~\eqref{eq:2nd_mmnt_product_summands} is fully determined by the~behaviour of $\tilde{\activitySymbol}_t^{(\depthSymbol)}\atWidth, t = 1, 2$, as $\sequenceVariable \to \infty$. We can thus focus on individual entries of the~expectation on the~RHS of the~above equation. For entry $(k, l)$ with $k, l \in \{1, \ldots, |\projectionIndeces|\}$, we have
\begin{align}\label{eq:modif_postnonlin_quad_form}
&\expectation{
	\tilde{\activitySymbol}_1^{(\depthSymbol)}\atWidth_i \,
	\tilde{\activitySymbol}_2^{(\depthSymbol)}\atWidth_j \,
	\tilde{\activitySymbol}_1^{(\depthSymbol)}\atWidth_k \,
	\tilde{\activitySymbol}_2^{(\depthSymbol)}\atWidth_l
}{}
\nonumber \\
&=
\delta_{(i) = (k)}
\delta_{(j) = (l)}
\expectation{
	\indexedActivity{\depthSymbol-1}{1}{\inputSymbol_{(i)}}\atWidth
	\indexedActivity{\depthSymbol-1}{2}{\inputSymbol_{(j)}}\atWidth
	\indexedActivity{\depthSymbol-1}{1}{\inputSymbol_{(k)}}\atWidth
	\indexedActivity{\depthSymbol-1}{2}{\inputSymbol_{(l)}}\atWidth
}{}
\, .
\end{align}

In analogy with the~proof of Lemma~\ref{lemma:conditional_limit_std}, we can establish convergence in distribution of the~four-way product inside the~RHS expectation, and combine Lemma~\ref{lemma:integrability_postnonlin} with Theorem~3.5 in \citep{billingsey68} to get convergence in distribution as $\sequenceVariable \to \infty$. Hence $\expectation{| \summand^{(\depthSymbol)}_{1}(\projectionIndeces,\projectionCoefficients)\atWidth \summand^{(\depthSymbol)}_{2}(\projectionIndeces,\projectionCoefficients)\atWidth |^2}{}$ converges to a~limit which is a~function of terms
\begin{align*}
&\expectation{
	\indexedActivity{\depthSymbol-1}{1}{\inputSymbol_{(i)}}\atLimit
	\indexedActivity{\depthSymbol-1}{2}{\inputSymbol_{(j)}}\atLimit
	\indexedActivity{\depthSymbol-1}{1}{\inputSymbol_{(k)}}\atLimit
	\indexedActivity{\depthSymbol-1}{2}{\inputSymbol_{(l)}}\atLimit
}{}
\nonumber \\
&=
\expectation{
	\indexedActivity{\depthSymbol-1}{1}{\inputSymbol_{(i)}}\atLimit
	\indexedActivity{\depthSymbol-1}{1}{\inputSymbol_{(k)}}\atLimit
}{}
\expectation{
	\indexedActivity{\depthSymbol-1}{2}{\inputSymbol_{(j)}}\atLimit
	\indexedActivity{\depthSymbol-1}{2}{\inputSymbol_{(l)}}\atLimit
}{}
\, .
\end{align*}
Substituting back and inspecting the~recursion in Lemma~\ref{lemma:limitVar} concludes the~proof.
\end{proof}

%

\begin{proof}[Proof of Lemma~\ref{lemma:3rd_abs_mmnt_summands}]
By H{\" o}lder's inequality, it is sufficient to exhibit a bound on the sequence of fourth moments, which are algebraically convenient to work with. Hence it is sufficient to prove that
\begin{equation*}
\expectation{
	\left|
		\summand^{(\depthSymbol)}_{1}(\projectionIndeces,\projectionCoefficients)\atWidth
	\right|^4
}{}
=
\projectionCoefficients^\top
\expectation{
	\tilde{\activitySymbol}_1^{(\depthSymbol)}\atWidth 
	\tilde{\activitySymbol}_1^{(\depthSymbol)}\atWidth^\top 
	\projectionCoefficients \projectionCoefficients^\top
	\tilde{\activitySymbol}_1^{(\depthSymbol)}\atWidth 
	\tilde{\activitySymbol}_1^{(\depthSymbol)}\atWidth^\top 
}{}
\projectionCoefficients
\, ,
\end{equation*}
is bounded by a~constant independent of $\sequenceVariable$. A~way to obtain such constant is to bound each term inside the~RHS expectation. We rearrange
\begin{equation*}
\expectation{
	\tilde{\activitySymbol}_1^{(\depthSymbol)}\atWidth 
	\tilde{\activitySymbol}_1^{(\depthSymbol)}\atWidth^\top 
	\projectionCoefficients \projectionCoefficients^\top
	\tilde{\activitySymbol}_1^{(\depthSymbol)}\atWidth 
	\tilde{\activitySymbol}_1^{(\depthSymbol)}\atWidth^\top 
}{}
=
\sum_{i=1}^{|\projectionIndeces|} \sum_{j=1}^{|\projectionIndeces|}
\projectionCoefficients_i \projectionCoefficients_j
\expectation{
	\tilde{\activitySymbol}_1^{(\depthSymbol)}\atWidth_i \,
	\tilde{\activitySymbol}_1^{(\depthSymbol)}\atWidth_j \,
	\tilde{\activitySymbol}_1^{(\depthSymbol)}\atWidth \,
	\tilde{\activitySymbol}_1^{(\depthSymbol)}\atWidth^\top
}{}
\, .
\end{equation*}
Hence it is sufficient to ensure that the~expectations
\begin{equation*}
\expectation{
	\tilde{\activitySymbol}_1^{(\depthSymbol)}\atWidth_i \,
	\tilde{\activitySymbol}_1^{(\depthSymbol)}\atWidth_j \,
	\tilde{\activitySymbol}_1^{(\depthSymbol)}\atWidth_k \,
	\tilde{\activitySymbol}_1^{(\depthSymbol)}\atWidth_l
}{}
\, ,
\end{equation*}
are bounded by a~constant independent of $\sequenceVariable$ for any combination of $i, j, k, l \in \{1, \ldots, |\projectionIndeces|\}$. Substituting back for $\tilde{\activitySymbol}_1^{(\depthSymbol)}\atWidth$
\begin{align*}
&\expectation{
	\tilde{\activitySymbol}_1^{(\depthSymbol)}\atWidth_i \,
	\tilde{\activitySymbol}_1^{(\depthSymbol)}\atWidth_j \,
	\tilde{\activitySymbol}_1^{(\depthSymbol)}\atWidth_k \,
	\tilde{\activitySymbol}_1^{(\depthSymbol)}\atWidth_l
}{}
\nonumber \\
&=
\delta_{(i) = (j) = (k) = (l)}
\expectation{
	\indexedActivity{\depthSymbol-1}{1}{\inputSymbol_{(i)}}\atWidth
	\indexedActivity{\depthSymbol-1}{1}{\inputSymbol_{(j)}}\atWidth
	\indexedActivity{\depthSymbol-1}{1}{\inputSymbol_{(k)}}\atWidth
	\indexedActivity{\depthSymbol-1}{1}{\inputSymbol_{(l)}}\atWidth
}{}
\, ,
\end{align*}
We thus only need to bound the~second factor on the~RHS. After upper bounding by the~absolute value, we can use Lemma~\ref{remark:joint_moments_bound} to conclude it is sufficient to bound the~fourth moment of $\indexedActivity{\depthSymbol-1}{1}{\inputSymbol_{(t)}}\atWidth$ for $t = 1, \ldots , |\projectionIndeces|$.\footnote{This result can be obtained by allowing only $p_i = 0, 1$ in Lemma~\ref{remark:joint_moments_bound}.} Using the~linear envelope condition, we see
\begin{equation*}
\expectation{
	\left|
		\indexedActivity{\depthSymbol-1}{1}{\inputSymbol_{(t)}}\atWidth
	\right|^4
}{}
\leq
2^{4 - 1}
\expectation{
	\envelopeconstant^4
	+
	\envelopegradient^4
	\left|
		\indexedActivation{\depthSymbol-1}{1}{\inputSymbol_{(t)}}\atWidth
	\right|^4
}{}
\, .
\end{equation*}
By Lemma~\ref{lemma:mmnt_bound_prenonlin} and a~simple application of H{\" o}lder's inequality, we know that the~fourth moment above is bounded by a~constant independent of $\sequenceVariable$. Because we are only considering a~finite set of inputs, we can bound the~fourth moments for all $\indexedActivation{\depthSymbol-1}{1}{\inputSymbol_{(t)}}\atWidth$ by a~shared constant, namely the~maximum over $t \in \{1, \ldots, |\projectionIndeces| \}$. This constant is independent of $\sequenceVariable$ which concludes the~proof.
\end{proof}

\subsection{Proofs of auxiliary results}

The~following results are useful in proving Lemmas~\ref{lemma:convg_product_2nd_mmnts} and~\ref{lemma:3rd_abs_mmnt_summands}.

\begin{lemma}\label{remark:joint_moments_bound} 
Suppose $\genericRV_1$, $\genericRV_2$, $\genericRV_3$, and $\genericRV_4$ are random variables on $\realLine$ with the~usual Borel $\sigma$-algebra. Assume that $\expectation{|\genericRV_i|^8}{} < \infty$ for all $i = 1, 2, 3, 4$. Then for any choice of $p_i = 0, 1, 2$ (where $i = 1, 2, 3, 4$), the expectations $\mathbb{E} [\prod_{i=1}^{4} |\genericRV_i|^{p_i}]$ are uniformly bounded by a polynomial in the $8$\textsuperscript{th} moments $\expectation{|\genericRV_i|^8}{} < \infty$ for $i=1,\ldots,4$.
\end{lemma}

\begin{proof}
Throughout this proof, we will be using the~following inequality
\begin{equation*}
\expectation{|\genericRV \genericRVB|}{}
\leq
\expectation{|\genericRV|}{} \expectation{|\genericRVB|}{}
+
\left\{
	\Var(|\genericRV|) \Var(|\genericRVB|)
\right\}^{1 / 2}
\, ,
\end{equation*}
which can be derived from the~boundedness of Pearson correlation coefficient.

Using the~above inequality, we have
\begin{equation*}\label{eq:4way_prod_ub}
\expectation{\prod_{i=1}^{4} |\genericRV_i|^{p_i}}{}
\leq
\expectation{|\genericRV_1|^{p_1} |\genericRV_2|^{p_2}}{}
\expectation{|\genericRV_3|^{p_3} |\genericRV_4|^{p_4}}{}
+
\left\{
	\Var[|\genericRV_1|^{p_1} |\genericRV_2|^{p_2}]
	\Var[|\genericRV_3|^{p_3} |\genericRV_4|^{p_4}]
\right\}^{1/2}
\, .
\end{equation*}
The~expectations in the~first term on the~RHS can then be again upper bounded, for example
\begin{equation*}
\expectation{|\genericRV_1|^{p_1} |\genericRV_2|^{p_2}}{}
\leq
\expectation{|\genericRV_1|^{p_1}}{}
\expectation{|\genericRV_2|^{p_2}}{}
+
\left\{
	\Var[|\genericRV_1|^{p_1}]
	\Var[|\genericRV_2|^{p_2}]
\right\}^{1/2}
\, ,
\end{equation*}
which is bounded if $\expectation{|\genericRV_i|^{4}}{} < \infty$ for $i = 1, 2$. Similarly for $\expectation{|\genericRV_3|^{p_3} |\genericRV_4|^{p_4}}{}$.

The~second term of the~first upper bound can be upper bounded in similar way
\begin{equation*}
\Var[|\genericRV_1|^{p_1} |\genericRV_2|^{p_2}]
\leq
\expectation{|\genericRV_1|^{2 p_1} |\genericRV_2|^{2 p_2}}{}
\, ,
\end{equation*}
where $\expectation{|\genericRV_1|^{2 p_1} |\genericRV_2|^{2 p_2}}{}$ can again be upper bounded by argument similar to the~above, yielding an upper bound that may be expressed as a fixed polynomial in $\expectation{|\genericRV_i|^{8}}{}$ for $i = 1, 2, 3, 4$, as $p_i \leq 2$ and the~lower order absolute moments may be bounded by exponents of the~higher ones via H{\" o}lder's inequality.
\end{proof}

\begin{lemma}\label{remark:exp_power_of_gauss_norm}
Assume $\weightSymbol_1, \ldots, \weightSymbol_k \in \realLine$ are arbitrary constants, and $\varepsilon_i$, $i = 1, \ldots, k$, are i.i.d.\ standard normal variables. Define the vector $w = (w_i)_{i=1}^k$. Then for $p \geq 0$
\begin{equation*}
\expectation{
	\left|
		\sum_{i=1}^{k}
			\weightSymbol_i \varepsilon_i
	\right|^p
}{}
=
\|w\|_2^p
\frac{2^{\tfrac{p}{2}} \Gamma(\tfrac{p + 1}{2})}{\Gamma(\tfrac{1}{2})}
\, .
\end{equation*}
\end{lemma}

\begin{proof}
Use the~linearity of the~dot product and Gaussianity of $\varepsilon_i$'s to obtain
\begin{equation*}
\expectation{
	\left|
		\sum_{i=1}^{k}
			\weightSymbol_i \varepsilon_i
	\right|^p
}{}
=
\expectation{
	\left|
		\|w\|_2 \tilde{\varepsilon}
	\right|^p
}{}
=
\|w\|_2^p \,
\expectation{
	\left|
		\tilde{\varepsilon}
	\right|^p
}{}
\, ,
\end{equation*}
where $\tilde{\varepsilon}$ is a~standard normal random variable. The~result is then obtained by realising that powers of standard normal are distributed according to Generalised Gamma variable for which the~expectation is known.
\end{proof}

\begin{lemma}\label{lemma:mmnt_bound_prenonlin} 
For any given $\depthSymbol \in \{ 1,2,\ldots,\depth + 1 \}$, and input $\inputSymbol \in \indexSet$, the~eighth moments of the~random variables $\indexedActivation{\depthSymbol}{\widthSymbol}{\inputSymbol}\atWidth$ are bounded by a~finite constant independent of $\sequenceVariable \in \naturalNumbers$ and $\widthSymbol \in \naturalNumbers$.
\end{lemma}

\begin{proof}
The~statement is trivially true for $\depthSymbol = 1$: the~law of $\indexedActivation{1}{\widthSymbol}{\inputSymbol}\atWidth$ for any $(\widthSymbol, \inputSymbol) \in \naturalNumbers \times \indexSet$ is a~normal distribution by the~Gaussianity of the~weights and biases, $\indexedActivation{1}{\widthSymbol}{\inputSymbol}\atWidth$ is equal in law to $\indexedActivation{1}{\widthSymbol}{\inputSymbol}[m]$, $\forall (m, \sequenceVariable) \in \naturalNumbers \times \naturalNumbers$, implying that the~moments are bounded by a~constant independent of $\sequenceVariable$, and independence of the~constant of index $\widthSymbol$ is obtained by exchangeability.

We can thus proceed by induction. We assume that the~condition holds for all $\depthSymbol = 1, 2, \ldots, t - 1$ (for some $t \in \{2,\ldots,\depth+1\}$), and prove that it must then also necessarily hold for $\depthSymbol = t$. First we obtain the following upper bound
\begin{equation*}
\expectation{|\indexedActivation{t}{\widthSymbol}{\inputSymbol}\atWidth|^8}{}
\leq
2^{8 - 1} \expectation{
	|\biasSymbol_i^{(t)}|^8
	+
	\left|
		\sum_{j=1}^{\widthFunction_{t-1}(\rowIndex)}
			\weightSymbol_{i, j}^{(t)} \activitySymbol_j^{(t-1)}(\inputSymbol)\atWidth
	\right|^8
}{}
\, ,
\end{equation*}
noting that the~expectation of the~first term is uniformly bounded in $\widthSymbol$ by properties of the~Gaussian distribution. 

Hence we focus on the~second term. We use Lemma~\ref{remark:exp_power_of_gauss_norm} to obtain
\begin{align}
\expectation{
	\left|
		\sum_{j=1}^{\widthFunction_{t-1}(\rowIndex)}
			\weightSymbol_{i, j}^{(t)} \activitySymbol_j^{(t-1)}(\inputSymbol)\atWidth
	\right|^8
}{}
&=
\expectation{
	\expectation{
		\left|
			\sum_{j=1}^{\widthFunction_{t-1}(\rowIndex)}
				\weightSymbol_{i, j}^{(t)} \activitySymbol_j^{(t-1)}(\inputSymbol)\atWidth
		\right|^8
		\, \Bigg| \,
		\activitySymbol_{1:\widthFunction_{t - 1}(\rowIndex)}^{(t-1)}(\inputSymbol)\atWidth
	}{}
}{}
\nonumber \\
&=
\frac{2^4 \Gamma(4 + 1 / 2)}{\Gamma(1 / 2)} \,
\expectation{
	\left|
		\frac{\weightVarianceScaled^{(t-1)}}{\widthFunction_{t-1}(\rowIndex)}
		|| \activitySymbol_{1:\widthFunction_{t-1}(\rowIndex)}^{(t-1)}(\inputSymbol)\atWidth ||_2^2
	\right|^4
}{} \label{eq:lemma16:boundme}
\, ,
\end{align}
where $\activitySymbol_{1:\widthFunction_{t-1}(\rowIndex)}^{(t-1)}(\inputSymbol)\atWidth$ is the~set of post-nonlinearities corresponding to $\sequenceIndex = 1, 2, \ldots, \widthFunction_{t-1}(\rowIndex)$. Observe that
\begin{align*}
\frac{1}{\widthFunction_{t-1}(\rowIndex)}
|| \activitySymbol_{1:\widthFunction_{t-1}(\rowIndex)}^{(t-1)}(\inputSymbol)\atWidth ||_2^2
&=
\frac{1}{\widthFunction_{t-1}(\rowIndex)}
\sum_{j=1}^{\widthFunction_{t-1}(\rowIndex)}
	(\activitySymbol_j^{(t-1)}(\inputSymbol)\atWidth)^2 \\
&\leq
\frac{1}{\widthFunction_{t-1}(\rowIndex)}
\sum_{j=1}^{\widthFunction_{t-1}(\rowIndex)}
	(\envelopeconstant + \envelopegradient | \activationSymbol_j^{(t-1)}(\inputSymbol)\atWidth |)^2
\, ,
\end{align*}
by the~linear envelope property. Suppressing a multiplicative constant independent of $\inputSymbol$ and $\sequenceVariable$ and substituting this bound back into Expression \eqref{eq:lemma16:boundme} yields
\begin{align*}
&\expectation{
	\left|
	\sum_{j=1}^{\widthFunction_{t-1}(\rowIndex)}
	\weightSymbol_{i, j}^{(t)} \activitySymbol_j^{(t-1)}(\inputSymbol)\atWidth
	\right|^8
}{} 
\nonumber \\
\leq &
\frac{1}{\widthFunction_{t-1}(\rowIndex)^4}
\expectation{
	\left|
		\sum_{j=1}^{\widthFunction_{t-1}(\rowIndex)}
			\envelopeconstant^2 +
			2 \envelopeconstant \envelopegradient | \activationSymbol_j^{(t-1)}(\inputSymbol)\atWidth | +
			\envelopegradient^2 | \activationSymbol_j^{(t-1)}(\inputSymbol)\atWidth |^2
	\right|^4
}{} \, .
\end{align*}
The~above can be simply multiplied out, yielding a~weighted sum of expectations of the~form
\begin{equation*}
\expectation{
	| \activationSymbol_k^{(t-1)}(\inputSymbol)\atWidth |^{p_1} \,
	| \activationSymbol_l^{(t-1)}(\inputSymbol)\atWidth |^{p_2} \,
	| \activationSymbol_r^{(t-1)}(\inputSymbol)\atWidth |^{p_3} \,
	| \activationSymbol_q^{(t-1)}(\inputSymbol)\atWidth |^{p_4}
}{} \, , 
\end{equation*}
with $p_i \in \{0, 1, 2\}$ for $i = 1, 2, 3, 4$, and $k,l,r,q \in \{1,\ldots,\widthFunction_{t-1}(\sequenceVariable)\}$, and where the weights of these terms are independent of $\sequenceVariable$. Using Lemma~\ref{remark:joint_moments_bound}, each of these terms is bounded if the~eighth moments of $\activationSymbol_k^{(t-1)}(\inputSymbol)\atWidth$ are bounded which is our inductive hypothesis. The~number of terms in the~expanded sum is upper bounded by $(3 \widthFunction_{t-1}(\rowIndex))^4$ and thus we can use the~same constant for any $\sequenceVariable \in \naturalNumbers$ due to the~$1 / \widthFunction_{\depthSymbol}(\rowIndex)^4$ scaling. Noticing that $\activationSymbol_j^{(t-1)}(\inputSymbol)\atWidth$ are exchangeable over the~index $j$ for any fixed $\inputSymbol$ and $\sequenceVariable$ concludes the~proof.
\end{proof}

\begin{lemma}\label{lemma:integrability_postnonlin}  
Consider a~collection of random variables $\activitySymbol_{i}^{(\depthSymbol)}(\inputSymbol_1)\atWidth$, $\activitySymbol_{j}^{(\depthSymbol)}(\inputSymbol_2)\atWidth$, $\activitySymbol_{k}^{(\depthSymbol)}(\inputSymbol_3)\atWidth$, and $\activitySymbol_{l}^{(\depthSymbol)}(\inputSymbol_4)\atWidth$ with any $i, j, k, l \in \naturalNumbers$, $\inputSymbol_1,\allowbreak \inputSymbol_2,\allowbreak \inputSymbol_3,\allowbreak \inputSymbol_4 \in \indexSet$, neither necessarily distinct. Then the~family of random variables
\begin{equation}\label{eq:postnonlin_4prod}
\activitySymbol_{i}^{(\depthSymbol)}(\inputSymbol_1)\atWidth
\activitySymbol_{j}^{(\depthSymbol)}(\inputSymbol_2)\atWidth
\activitySymbol_{k}^{(\depthSymbol)}(\inputSymbol_3)\atWidth
\activitySymbol_{l}^{(\depthSymbol)}(\inputSymbol_4)\atWidth
\, ,
\end{equation}
indexed by $\sequenceVariable$ is uniformly integrable for any $\depthSymbol = 1, 2, \ldots, \depth + 1$.
\end{lemma}

\begin{proof}
A~simple way to prove uniform integrability of an~arbitrary family of real-valued random variables $\{\genericRV_n\}_{n \in \naturalNumbers}$ is to show $\sup_n \expectation{|\genericRV_n|^{1 + \epsilon}}{} < \infty$ for some $\epsilon > 0$. We use $\epsilon = 1$, i.e.\ the~second moment of the~four-way product from Equation~\eqref{eq:postnonlin_4prod} will be bounded by a~constant independent of $\sequenceVariable$.

Write
\begin{equation*}
\expectation{
	\left|
		\activitySymbol_{i}^{(\depthSymbol)}(\inputSymbol_1)\atWidth
		\activitySymbol_{j}^{(\depthSymbol)}(\inputSymbol_2)\atWidth
		\activitySymbol_{k}^{(\depthSymbol)}(\inputSymbol_3)\atWidth
		\activitySymbol_{l}^{(\depthSymbol)}(\inputSymbol_4)\atWidth
	\right|^2
}{}
\, ,
\end{equation*}
and recall that by Lemma~\ref{remark:joint_moments_bound}, we only need to bound the~eighth moments of $\activitySymbol_{i}^{(\depthSymbol)}(\inputSymbol)\atWidth$ by a~constant independent of $\sequenceVariable$ for $i, j, k$ and $l$, and $\{\inputSymbol_t\}_{t=1}^4$. Using the~linear envelope property
\begin{equation*}
\expectation{
	\left|
		\indexedActivity{\depthSymbol}{i}{\inputSymbol}\atWidth
	\right|^8
}{}
\leq
2^{8 - 1}
\expectation{
	\envelopeconstant^8
	+
	\envelopegradient^8
	\left|
		\indexedActivation{\depthSymbol}{i}{\inputSymbol}\atWidth
	\right|^8
}{}
\, .
\end{equation*}
The~result in Lemma~\ref{lemma:mmnt_bound_prenonlin} gives us a~constant upper bounding the~left hand side which depends on $\inputSymbol$ but not $\sequenceVariable$. Because we are considering only a~fixed finite set of inputs, we can take the~maximum of these constants to conclude the~proof.
\end{proof}

\end{document}